\documentclass[11pt]{article}

\usepackage[
  shownumpages,               
  bgcolor={245,245,250},      
  braincolor={blue},    
  citingstyle=authoryear,     
  bibliostyle=plainnat,       
  bibfile=references          
]{./styles/brainlab}

\usepackage{hyperref}       
\usepackage{url}            
\usepackage{booktabs}       
\usepackage{amsfonts}       
\usepackage{microtype}      
\usepackage{xcolor}         

\usepackage{amsmath}
\usepackage{amssymb}
\usepackage{mathtools}
\usepackage{amsthm}
\usepackage{wrapfig}
\usepackage{nicefrac}
\allowdisplaybreaks

\usepackage{multirow}
\usepackage{threeparttable}
\usepackage{cleveref}

\usepackage{makecell}

\usepackage{algorithm}
\usepackage{algpseudocode}

\newcommand{\lmo}{\text{lmo}}
\newcommand{\norm}[1]{{\left \| #1 \right\|}}
\newcommand{\R}{\mathbb{R}}
\newcommand{\el}{\mathcal{L}}
\newcommand{\eqdef}{\stackrel{\text{def}}=}
\DeclareMathOperator*{\argmax}{argmax}
\usepackage[table]{xcolor}
\definecolor{bgcolor2}{rgb}{0.8,1,0.8}

\setbrainmeta{
  title={Preconditioned Norms: A Framework Unifying Steepest Descent, Quasi-Newton and Adaptive Methods},
  authors={
    Andrey Veprikov\textsuperscript{1, 2},
    Arman Bolatov\textsuperscript{2},
    Aleksandr Bogdanov\textsuperscript{1},
    Samuel Horváth\textsuperscript{2}, 
    Aleksandr Beznosikov\textsuperscript{1, 3}, 
    Martin Takáč\textsuperscript{2}, 
    Slavomir Hanzely\textsuperscript{2}
  },
  affiliations={
    \textsuperscript{1}Basic Research of Artificial Intelligence Laboratory (BRAIn Lab)\\
    \textsuperscript{2}Mohamed bin Zayed University of Artificial Intelligence (MBZUAI)\\
    \textsuperscript{3}Innopolis University\\
  },
  abstract={
    Optimization lies at the core of modern deep learning, yet existing methods often face a fundamental trade-off between adapting to problem geometry and leveraging curvature utilization. Steepest descent algorithms adapt to different geometries through norm choices but remain strictly first-order, whereas quasi-Newton and adaptive optimizers incorporate curvature information but are restricted to Frobenius geometry, limiting their applicability across diverse architectures. In this work, we propose a unified framework generalizing steepest descent, quasi-Newton methods, and adaptive methods through the novel notion of preconditioned matrix norms. This abstraction reveals that widely used optimizers such as SGD and Adam, as well as more advanced approaches like Muon and KL-Shampoo, and recent hybrids including SOAP and SPlus, all emerge as special cases of the same principle. Within this framework, we provide the first systematic treatment of affine and scale invariance in the matrix-parameterized setting, establishing necessary and sufficient conditions under generalized norms. Building on this foundation, we introduce two new methods, $\texttt{MuAdam}$ and $\texttt{MuAdam-SANIA}$, which combine the spectral geometry of Muon with Adam-style preconditioning. Our experiments demonstrate that these optimizers are competitive with, and in some cases outperform, existing state-of-the-art methods. Our code is available at \url{https://github.com/brain-lab-research/LIB/tree/quasi_descent}
  },
}

\usepackage[textsize=tiny]{todonotes}

\begin{document}
\begin{mainpart}

\section{Introduction}

Optimization lies at the heart of modern machine learning \citep{bottou2010large, goodfellow2016deep, team2025kimi}, including today’s most prominent systems such as Large Language Models \citep{vaswani2017attention, brown2020language, hernandez2025apertus} and generative AI \citep{goodfellow2014generative, rombach2022high}. 
Improvements in optimization efficiency or stability directly translate into faster training, reduced computational costs, and ultimately better-performing models \citep{kingma2014adam}. Training a model amounts to adjusting its parameters $W$ to minimize a loss function $\mathcal{L}(W)$, which measures the discrepancy between predictions and data. While this formulation has long been fundamental to learning theory and practice, deep learning introduces new challenges: datasets are very large \citep{dean2012large, tian2025survey}, parameter spaces are extremely high-dimensional \citep{team2025kimi, hernandez2025apertus}, and the loss landscapes are highly non-convex \citep{choromanska2015loss, chen2025understanding}.

As computing the full gradient is computationally prohibitive, the standard approach to these challenges is \emph{stochastic optimization} \citep{robbins1951stochastic, rumelhart1986learning, tian2023recent}. Instead of computing the full loss, training samples $\xi \sim \mathcal{D}$ are drawn from the underlying data distribution, and the expected loss $$\mathcal{L}(W) = \mathbb{E}_{\xi \sim \mathcal{D}} \mathcal{L}(W;\xi)$$
is minimized using stochastic gradients $G_t = \nabla \mathcal{L}(W_t; \xi_t)$ computed on samples or minibatches \citep{bottou2010large, shalev2014understanding, ward2022stochastic}.  
The model parameters are then updated iteratively \citep{rumelhart1986learning, bernstein2024old}:
\begin{equation}
    \label{eq:W_t+1}
    W_{t+1} = W_t - \alpha_t \Delta W_t,
\end{equation}
where $\alpha_t$ is the learning rate and the update rule $\Delta W_t$ depends on the chosen optimization method. For example, Stochastic Gradient Descent (SGD) corresponds to $\Delta W_t = G_t$ \citep{rumelhart1986learning}. This general template encompasses most practical algorithms, ranging from Muon \citep{jordan2024muon} to quasi-Newton methods \citep{gupta2018shampoo} and adaptive optimizers \citep{kingma2014adam}.

Traditional optimization approaches typically vectorize all parameters, treating them as elements in the space $\mathbb{R}^{mn}$ \cite{rumelhart1986learning}. Recent work emphasizes that in practice, parameters $W$ possess natural \emph{matrix structure}: $W \in \mathbb{R}^{m \times n}$ (e.g., the weights of linear or convolutional layers), which can be exploited to achieve faster and more robust convergence \citep{gupta2018shampoo, goldfarb2020practical, bernstein2024modular, jordan2024muon, vyas2024soap, pethick2025training, riabinin2025gluon}.

A useful perspective on update rules is the classical \emph{steepest descent} interpretation \citep{bernstein2024old, pethick2025training}: the update direction corresponds to the direction of steepest loss decrease under a chosen norm, where the choice of norm defines the optimization geometry. For vector-valued parameters, $\ell_2$ steepest descent recovers normalized SGD \citep{hazan2015beyond}, while $\ell_\infty$ steepest descent yields SignSGD \citep{bernstein2018signsgd}. Recently, for matrix-valued parameters, Muon-like algorithms have employed the spectral (singular value) matrix norm \citep{jordan2024muon, pethick2025training, riabinin2025gluon}, which explicitly leverages the structural properties of weight matrices. The steepest descent viewpoint therefore unifies many modern optimizers. However, these algorithms do not incorporate second-order curvature information, which naturally motivates quasi-Newton-based techniques \citep{liu1989limited, gupta2018shampoo}.

Quasi-Newton methods approximate curvature by introducing a positive definite matrix $H_t \succeq 0$ that estimates the Hessian $\nabla^2 \mathcal{L}(W_t)$ \citep{goldfarb2020practical, gao2024gradient}. In the vector case, they produce preconditioned updates of the form $\Delta W_t = H_t^{-1} G_t$
\citep{davidon1959variable,broyden1970convergence, fletcher1970new, goldfarb1970family, shanno1970conditioning, broyden1967quasi, liu1989limited}.
Beyond faster convergence, quasi-Newton methods are popular due to their natural geometric properties—they preserve the sequence of iterates regardless of function scaling or choice of basis (the so-called \emph{affine invariance} \citep{nesterov1994interior, nesterov2018lectures, d2018optimal}), facilitating implementation and hyperparameter tuning. 
In the matrix case, quasi-Newton updates naturally generalize to:
\begin{equation}
\label{eq:quasi_newton_intro}
    \Delta W_t = (H_t^{L})^{-1} \cdot G_t \cdot (H_t^{R})^{-1},
\end{equation}
where $H_t^L$ and $H_t^R$ approximate left and right curvature factors, and their Kronecker product acts as a surrogate Hessian $H_t$. This principle underlies widely used methods such as K-FAC \citep{martens2015optimizing}, Shampoo \citep{gupta2018shampoo}, and SOAP \citep{vyas2024soap}, which exploit layer-wise gradient covariances for computational scalability. 

Another important approach to Hessian estimation involves element-wise preconditioners such as AdaGrad \citep{duchi2011adaptive}, RMSProp \citep{tieleman2012rmsprop}, and Adam \citep{kingma2014adam}, which, in the vector case, can be interpreted as quasi-Newton updates with diagonal $H_t$ and possess the property of \emph{scale invariance} \citep{abdukhakimov2023sania, choudhury2024remove}. By contrast, adaptive methods in the matrix domain use the Hadamard product, yielding updates of the form
\begin{equation}
\label{eq:adaptive_intro}
    \Delta W_t = V_t^{\circ-1} \odot G_t,
\end{equation}
where $V_t$ is a matrix with positive elements (scaling factors), $\odot$ denotes the Hadamard product, and $V_t^{\circ -1}$ denotes the element-wise inverse. These are also referred to as \emph{adaptive methods}, and due to their element-wise nature, they enjoy scale invariance rather than affine invariance.

Together, quasi-Newton and element-wise preconditioned methods constitute the two main paradigms for incorporating curvature information into deep learning optimization. While these approaches can achieve important geometric properties such as affine or scale invariance, they remain fundamentally constrained to the Frobenius norm, limiting their ability to capture more complex geometric structures. In contrast, steepest descent methods offer flexibility in norm selection but inherently lack these crucial invariance properties, thereby restricting their effectiveness when optimizing across the diverse structural landscapes of modern deep learning architectures. These limitations motivate the central question of our work:
\begin{center}
    \emph{Can optimization algorithms inherit both the geometric adaptability of steepest descent \newline and the curvature-awareness of quasi-Newton and adaptive approaches?}
\end{center}
Our work provides a positive answer to this question, with the following main contributions:
\begin{enumerate}
    \item \textbf{Unification of optimization methods through generalized norms.} We develop a comprehensive framework based on generalized norms (Definitions~\ref{def:norm_L_R} and \ref{def:norm_D}) that reveals fundamental connections between seemingly disparate optimization approaches. Our framework demonstrates that classical steepest descent, quasi-Newton methods, and adaptive element-wise algorithms are special cases of the same underlying principle, while also providing a principled interpretation of recently proposed methods like SOAP \citep{vyas2024soap} and SPlus \citep{frans2025stable}.

    \item \textbf{Systematic optimizer design methodology.} We derive the steepest descent update for arbitrary generalized norms (Theorem \ref{thm:lmo_general}), establishing a principled framework for designing new optimizers. Our approach enables the creation of novel algorithms by combining different curvature approximations with various descent geometries, as demonstrated with our newly proposed optimizer (Algorithm \ref{alg:muadam}).

    \item \textbf{First systematic analysis of invariance properties in matrix-parametrized optimization.} We provide the first comprehensive study of affine and scale invariance in the matrix-parametrized setting, deriving explicit necessary and sufficient conditions within our framework (\Cref{theorem:affine_invariance}).

    \item \textbf{Empirical validation of proposed methods and invariance properties.} We demonstrate the practical effectiveness of our framework through extensive experiments, showcasing the performance advantages of our proposed optimizers and empirically verifying their theoretical invariance properties (Section \ref{ssec:experiments}).
\end{enumerate}

The remainder of the paper is structured as follows. In the next section, we define technical notation and discuss related work. In \Cref{ssec:framework}, we present our unifying framework, where we introduce generalized norms and their connections to the literature, and derive the corresponding update steps. In \Cref{ssec:geometric}, we establish necessary and sufficient conditions for affine and scale invariance in the matrix-parametrized case. Finally, in \Cref{ssec:experiments}, we present our numerical evaluation.
\subsection{Preliminaries}\label{ssec:preliminaries}

The steepest descent principle provides one of the most classical formulations of first-order optimization \citep{bernstein2024modular, bernstein2024old}. In this framework, each update direction is defined as the solution of a norm-constrained quadratic model of the loss. More recently, this perspective has been revisited in deep learning, where the update step is written in terms of a \emph{Linear Minimization Oracle} (LMO) \citep{pethick2025training}. Concretely, given the matrix $G_t \in \R^{m\times n}$, the LMO is defined as
\begin{equation}
\label{eq:lmo}
    \lmo(G_t) \in \argmax_{T\in \R^{m\times n}:\|T\|\le \rho} \langle G_t, T \rangle ,
\end{equation}
where $\norm{\cdot}$ is a matrix norm, $\rho>0$ is a scaling parameter, and $\langle \cdot,\cdot\rangle$ denotes the Frobenius inner product. The corresponding steepest descent update then takes the form
\begin{equation}
\label{eq:lmo_step}
    \Delta W_t = \lmo(G_t) .
\end{equation}

A general and widely studied family is given by the $\alpha \to \beta$ operator norms:
\[
   \| G \|_{\alpha\to\beta} = \sup_{\|x\|_\alpha \le 1} \| Gx \|_\beta ,
\]
where $\norm{\cdot}_\alpha$ and $\norm{\cdot}_\beta$ are vector norms. 
A first important case  uses the root‑mean‑square norm, defined for 
$x \in \mathbb{R}^d$ as 
$
   \|x\|_{\mathrm{RMS}} := \text{dim}(x)^{-\frac 12}  \|x\|_2 .
$
When $\alpha = \beta = \text{RMS}$, 
the resulting operator norm coincides with the spectral norm and the corresponding steepest descent method is known as Muon \citep{jordan2024muon}. 
Here, the core step is the spectral LMO, which  for gradients with 
singular value decomposition $G_t = U_t \Sigma_t V_t^\top$ returns $U_t V_t^\top$.  
Muon avoids computing the explicit SVD decomposition via Newton–Schulz iteration \citep{bernstein2024old}, 
an efficient procedure based only on matrix multiplications, 
to approximate this projection onto the spectral geometry.
Beyond the spectral case, other choices of $\alpha$ and $\beta$ yield different but related 
update rules. Two further notable instances are column‑normalized and row‑normalized steps using the $1 \to \mathrm{RMS}$ and  $\mathrm{RMS} \to \infty$ norms \citep{pethick2025training}.

These constructions illustrate how steepest descent can naturally incorporate 
matrix geometry, but they remain fundamentally tied to first‑order information, 
as no curvature or higher‑order structure of the loss $\mathcal{L}$ is explicitly 
taken into account. This limitation motivates the quasi‑Newton and adaptive approaches 
discussed in the next section.

\subsection{Quasi-Newton and Adaptive Methods}

In the matrix setting, quasi-Newton updates are usually expressed through two sided preconditioning,
\begin{equation}
\label{eq:qn_matrix}
    \Delta W_t = (L_t^\top L_t)^{-1}  G_t  (R_t^\top R_t)^{-1},
\end{equation}
where, in contrast to \eqref{eq:quasi_newton_intro}, we use matrices $L_t$ and $R_t$, chosen such that 
$L_t^\top L_t := H^L_t$ and $R_t^\top R_t := H^R_t$. 
This parametrization is purely for convenience (see Definition \ref{def:norm_L_R} and Theorem \ref{thm:lmo_general}).

Several choices have been studied in the literature:
\begin{align}
    &H_t^L = \sum_{s=0}^{t-1} G_s G_s^\top,\qquad\qquad\qquad~~
    H_t^R = \sum_{s=0}^{t-1} G_s^\top G_s, 
    \tag{Shampoo \citep{gupta2018shampoo}} \\
    &H_t^L = (1 - \beta) H_{t-1}^L + \beta G_s G_s^\top,~~~
    H_t^R = (1 - \beta) H_{t-1}^R + \beta G_s^\top G_s, 
    \tag{SOAP \citep{vyas2024soap}}
\end{align}
with $\beta \in (0,1)$ controlling exponential averaging. These constructions approximately factorize the Hessian via the Kronecker structure $H_t = H_t^R \otimes H_t^L$ and are representative of the widely used \emph{Kronecker-factored preconditioning} family \citep{martens2015optimizing,zhang2025concurrence}.

Beyond such Kronecker-based methods, adaptive optimizers construct element-wise diagonal preconditioners. A general update can be written as
\begin{equation}
\label{eq:adaptive_matrix}
    \Delta W_t = (D_t \odot D_t)^{\circ -1} \odot G_t ,
\end{equation}
where $D_t \in \R^{m\times n}$ stores positive coordinate-wise statistics, $\circ-1$ denotes element-wise 
inversion, and $\odot$ is the Hadamard product. 
Analogously to the quasi-Newton case
here we use $D_t \odot D_t := V_t$ instead of $V_t$ as in \eqref{eq:adaptive_intro} for convenience. Notable examples of adaptive preconditioners include:
\begin{align}
    &V_t = \left(\sum_{s=0}^{t-1} G_s \odot G_s \right)^{\circ \frac 12}, 
    \tag{AdaGrad \citep{duchi2011adaptive}} \\
    &V_t = \left(\beta V_{t-1} + (1-\beta) (G_t \odot G_t) \right)^{\circ \frac 12} .
    \tag{Adam \citep{kingma2014adam}}
\end{align}

Thus, while quasi-Newton methods exploit Kronecker-factored approximations of curvature through $H_t^L, H_t^R$, adaptive methods rely on element-wise preconditioners $V_t$ derived from gradient magnitudes. 
These formulations illustrate the range of preconditioning strategies used in deep learning, and serve as key reference points for the unified framework proposed in this work.

\paragraph{Relation to prior work on spectral norm optimization.} While quasi-Newton and adaptive methods have predominantly been developed within the Frobenius norm framework, several works have explored integrating preconditioning with alternative geometries. \citet{carlson2015stochastic, carlson2015preconditioned, carlson2015stochasticgraphical} demonstrated the advantages of steepest descent methods under the $\ell_\infty$ and spectral (Schatten-$\infty$) norms by introducing stochastic spectral descent method. \citet{carlson2015preconditioned} specifically focused on adaptive optimization for neural networks, proposing RMSspectral and ADAspectral—adaptive methods that rescale gradients within a spectral geometry.
More recently, \citet{bernstein2024old} renewed interest in spectral descent by interpreting the popular Shampoo optimizer~\citep{gupta2018shampoo} as a momentum-based variant of stochastic spectral descent.
Our work generalizes this line of research by incorporating quasi-Newton curvature approximations through the $(L,R)$-norm (\Cref{def:norm_L_R}) and by interpreting their adaptive spectral methods of \citet{carlson2015preconditioned} as special cases of the $D$-norm formulation (\Cref{def:norm_D}). We strengthen this geometric interpretation by presenting the first systematic analysis of affine and scale invariance in matrix-parameterized optimization (\Cref{theorem:affine_invariance}). This enables principled algorithm design, which we leverage to introduce MuAdam-SANIA (\Cref{alg:muadam}), a scale-invariant optimization method that, similarly to Muon \citep{jordan2024muon},
avoids expensive randomized SVD computations by employing an efficient approximation based on the Newton–Schulz iteration of \citet{bernstein2024old}.

From a numerical standpoint, while \citet{carlson2015preconditioned} evaluated their methods primarily on vision tasks such as MNIST and CIFAR, we demonstrate the effectiveness of our approach in large-scale modern applications, including LLM fine-tuning and comprehensive NLP benchmarks (\Cref{ssec:experiments}). 
Finally, \citet{crawshaw2025exploration} complements this landscape by systematically evaluating admissible geometries (i.e., combinations of layerwise norms and product geometries) and introducing Momo-based model truncation for step-size adaptation.

\paragraph{Computation of \( UV^\top \).}
The problem of computing the product \( UV^\top \) from the SVD decomposition \( G = U \Sigma V^\top \) of a matrix \( G \) is a classical and well-studied topic in linear algebra~\citep{autonne1902groupes}. Over the decades, it has been referred to variously as the computation of the \emph{polar factor}, \emph{symmetric orthogonalization}, the \emph{matrix sign function}. \citet{kovarik1970some} and \citet{bjorck1971iterative} proposed computing the polar factor via iterative schemes, and \citet{higham1986computing} introduced a quadratically convergent Newton method.

In machine learning, where model parameters may be matrix-valued, using the polar factor of the gradient matrix as a descent direction can be viewed as a \emph{Stochastic Spectral Descent} method in a particular geometry \citep{carlson2015preconditioned}. This approach has become practical in modern deep learning thanks to GPU-based parallelism, which enables efficient matrix addition and multiplication~\citep{bernstein2024modular}. Its practical viability was demonstrated by \citet{jordan2024muon}, who proposed the Muon algorithm (MomentUm Orthogonalized by Newton--Schulz), showing improved performance over Adam for large-scale neural network training. However, despite referencing the Newton--Schulz algorithm, the actual implementation in \citet{jordan2024muon} (their ``NewtonSchulz5'' subroutine) does not match the standard Newton--Schulz iteration. 

As an alternative to Newton--Schulz, \citet{nakatsukasa2010optimizing} proposed the cubically convergent \emph{Dynamically Weighted Halley} (DWH) algorithm, along with a QR-based variant (QDWH), which is backward stable~\citep{nakatsukasa2012backward}. In a distinct line of development, \citet{amsel2025polar} introduced the worst-case optimal polynomial-based algorithm \emph{PolarExpress} approximating the polar factor.

\section{Novel Framework}\label{ssec:framework}

Our goal is to unify steepest descent, quasi-Newton, and adaptive methods under a single geometric framework. 
The central idea is to define new families of matrix norms that encode preconditioning directly, 
and then to characterize the associated LMOs, which determine the update steps. 

We first introduce two classes of norms that generalize both quasi-Newton and adaptive updates.

\begin{definition}
\label{def:norm_L_R}
For any matrix $G \in \R^{m \times n}$, positive definite matrices $L \in \R^{m \times m}, R \in \R^{n \times n}$, 
and a base matrix norm $\norm{\cdot}$, we call a \emph{$(L,R)$-preconditioned matrix norm},
$$
   \| G \|_{L,R,\norm{\cdot}} \eqdef \| L \cdot G \cdot R \| .
$$
\end{definition}
This general norm includes as a special case Kronecker-factored methods when the base norm is Frobenius ($ \norm{\cdot}_{L,R,\norm{\cdot}_F}$), which leads to the quasi-Newton update \eqref{eq:qn_matrix} from the LMO step \eqref{eq:lmo} \cite{li2024quasi}.

\begin{definition}
\label{def:norm_D}
For any matrix $G \in \R^{m \times n}$, diagonal positive matrix $D \in \R^{m \times n}$, 
and any matrix norm $\norm{\cdot}$, we call \emph{$D$-preconditioned matrix norm},
$$
   \| G \|_{D,\norm{\cdot}} \eqdef \| D \odot G\|.
$$
\end{definition}
This formulation includes adaptive optimizers with Frobenius base norm.

Together, these definitions capture the two major strands of preconditioning in deep learning: Kronecker-based curvature approximations via $(L,R)$ and coordinate-wise scaling via $D$. Note, that if all preconditioners are chosen as identity matrices, 
the framework reduces to the classical steepest descent updates, depending on the base norm $\norm{\cdot}$. 
Thus, many classical algorithms emerge as special cases of the generalized construction.

Linear minimization oracles associated with these norms have a simple structure: they reduce to the LMO of the underlying base norm in a transformed gradient space. To clarify technical details, let's explicitely denote the choice of the base norm in the LMOs an $\lmo_{\norm{\cdot}}$.

\begin{theorem}
\label{thm:lmo_general}
The linear minimization oracles for $(L,R)$-norm and $D$-norm can be expressed as\footnote{For an invertible matrix $M$, we use shorthand notation for the inverse transpose as $M^{-T}$.}
\begin{align*}
   \lmo_{L,R,\norm{\cdot}}(G) &= L^{-1}\lmo_{\norm{\cdot}}(L^{-T} G R^{-T})R^{-1},\\
   \lmo_{D,\norm{\cdot}}(G) &= D^{\circ -1} \odot \lmo_{\norm{\cdot}}(D^{\circ -1} \odot G).
\end{align*}
\end{theorem}

Theorem~\ref{thm:lmo_general} highlights that in the preconditioned setting the LMO acts within a transformed gradient space. For the $(L,R)$-norms, the gradient is mapped to the transformed space $L^{-T} G R^{-T}$ where the base LMO is applied, and the result is mapped back to the original space by $L^{-1}$ and $R^{-1}$. This general characterization of LMO enables unification of seemingly different approaches -- norm‑constrained steepest descent, Kronecker‑factored quasi‑Newton methods, adaptive optimizers, and recent hybrids all emerge as special cases. We summarized this observations in Table~\ref{tab:optimization_methods}.
\begin{table}[h!]  
\renewcommand{\arraystretch}{1.5}
\captionof{table}{Popular optimization methods and their parameterization within the unified framework. Sign~``$-$'' indicates that the corresponding parameter is not used in the method. EMA indicates that the exponential moving average of the corresponding quantity is utilized.}
\label{tab:optimization_methods}   
\resizebox{\linewidth}{!}{
  \begin{threeparttable}
    \begin{tabular}{|c|c|c|c|c|c|}
    \cline{2-6}
    \multicolumn{1}{c|}{} & \textbf{Method} & 
    {$\mathbf{L_t}$}
    & 
    {$\mathbf{R_t}$}
    &
    {$\mathbf{D_t}$}
    &
    \textbf{Base Norm}
    \\
    \hline
    \multirow{5}{*}{\rotatebox[origin=c]{90}{\shortstack[c]{\textbf{Norm-based}}}} 
    & \texttt{Normalized SGD} \citep{hazan2015beyond} & $-$ & $-$ & $-$ & Frobenius \\ \cline{2-6}
    & \texttt{SignSGD} \citep{bernstein2018signsgd} & $-$ & $-$ & $-$ & $\ell_1 \to \ell_\infty$ \\ \cline{2-6}
    & \texttt{Muon} \citep{jordan2024muon} & $-$ & $-$ & $-$ & Spectral \\
    \cline{2-6}
    & \texttt{Scion} \citep{pethick2025training} & $-$ & $-$ & $-$ & \begin{tabular}{@{}c@{}}
        matrices: Spectral
        \\
        vectors: $\ell_{\infty}$
    \end{tabular} \\
    \hline
    \hline
    \multirow{4}{*}{\rotatebox[origin=c]{90}{\shortstack[c]{\textbf{Quasi} \\ \textbf{Newton}}}} 
    & \texttt{K-FAC} \citep{martens2015optimizing} & $(\mathbb{E}[GG^\top])^{1/8}$ & $(\mathbb{E}[G^\top G])^{1/8}$ & $-$ & Frobenius \\ \cline{2-6}
    & \texttt{Shampoo} \citep{gupta2018shampoo} & $(\sum_{s=1}^{t-1} G_s G_s^\top)^{1/8}$ & $(\sum_{s=1}^{t-1} G_s^\top G_s)^{1/8}$ & $-$ & Frobenius \\ \cline{2-6}
    & \texttt{One-sided Shampoo} \citep{xie2025structured} & $(\sum_{s=1}^{t-1} G_s G_s^\top)^{1/4}$ & $I$ & $-$ & Frobenius \\ \cline{2-6}
    & \texttt{KL-Shampoo} \citep{lin2025understanding} & $(\text{EMA}[G_t [R_t^TR_t]^{-1} G_t^\top])^{1/8}$ & $(\text{EMA}[G_t^\top [L_t^TL_t]^{-1} G_t])^{1/8}$ & $-$ & Frobenius \\ \cline{2-6}
    \hline
    \hline
    \multirow{4}{*}{\rotatebox[origin=c]{90}{\shortstack[c]{\textbf{Adaptive}}}} 
    & \texttt{AdaGrad} \citep{duchi2011adaptive} & $-$ & $-$ & $(\sum_{s=1}^{t-1} G_s \odot G_s)^{1/4}$ & Frobenius \\ \cline{2-6}
    & \texttt{Adam} \citep{kingma2014adam} & $-$ & $-$ & $(\text{EMA}[G_t \odot G_t])^{1/4}$ & Frobenius \\ \cline{2-6}
    & \texttt{MADGRAD} \citep{defazio2022momentumized} & $-$ & $-$ & $(\text{EMA}[G_t \odot G_t])^{1/6}$ & Frobenius \\ \cline{2-6}
    & \texttt{Adam-SANIA} \citep{abdukhakimov2023sania} & $-$ & $-$ & $(\text{EMA}[G_t \odot G_t])^{1/2}$ & Frobenius \\ 
    \hline
    \hline
    \multirow{4}{*}{\rotatebox[origin=c]{90}{\shortstack[c]{\textbf{Hybrid}}}} 
    & \texttt{SOAP} \citep{vyas2024soap} & $Q_L$ \tnote{{\color{blue} (1)}} & $Q_R$ \tnote{{\color{blue} (1)}} & $-$ & $\norm{\cdot}_{D=\text{Adam},~\norm{\cdot}_F}$\!\!\!\!\!\tnote{{\color{blue}(2)}} \\ \cline{2-6}
    & \texttt{SPlus} \citep{frans2025stable} & $Q_L$ \tnote{{\color{blue} (1)}} & $Q_R$ \tnote{{\color{blue} (1)}}& $-$ & $\ell_1 \to \ell_\infty$ \\ \cline{2-6}
    & \cellcolor{bgcolor2}{\texttt{MuAdam} (Algorithm \ref{alg:muadam} with $p=1/4$)} & \cellcolor{bgcolor2}{$-$} & \cellcolor{bgcolor2}{$-$} & \cellcolor{bgcolor2}{$(\text{EMA}[G_t \odot G_t])^{1/4}$}
    &
    \cellcolor{bgcolor2}{Spectral} \\ \cline{2-6}
    & \cellcolor{bgcolor2}{\texttt{MuAdam-SANIA} (Algorithm \ref{alg:muadam} with $p=1/2$)} & \cellcolor{bgcolor2}{$-$} & \cellcolor{bgcolor2}{$-$} & \cellcolor{bgcolor2}{$(\text{EMA}[G_t \odot G_t])^{1/2}$}
    &
    \cellcolor{bgcolor2}{Spectral} \\ \hline
    \end{tabular}     
    \renewcommand{\TPTnoteSettings}{\small} 
    \begin{tablenotes}
    {  
        \item[] 
        \small
        \tnote{{\color{blue}(1)}} $Q_L, Q_R$ are eigenbasis matrices from Shampoo's preconditioners $\sum_{s=1}^{t-1} G_s G_s^\top$ and $\sum_{s=1}^{t-1} G_s^\top G_s$ respectively.\\
        \tnote{{\color{blue}(2)}} $\norm{\cdot}_{D=\text{Adam},~\norm{\cdot}_F}$ denotes the $D$-norm with Adam diagonal preconditioning and Frobenius norm.
    }
    \end{tablenotes}
    \end{threeparttable}
}
\end{table}
Among the methods listed in \Cref{tab:optimization_methods}, hybrid methods SOAP \citep{vyas2024soap} and SPlus \citep{frans2025stable} deserve special attention. 
Both explicitly combine the geometry of quasi-Newton style preconditioning with 
a linear minimization oracle step: in SOAP, Shampoo’s Kronecker preconditioners \citep{gupta2018shampoo}
are coupled with Adam-style diagonal adaptation \citep{kingma2014adam} performed in the preconditioned eigenbasis, 
while in SPlus, the same Kronecker structure is combined with a sign-based LMO \citep{bernstein2018signsgd}. 
These optimizers exemplify precisely the type of integration that our theory highlights: 
the LMO does not replace the preconditioner but interacts with it in a structured way.  
Empirically, both SOAP and SPlus have shown strong performance on large-scale deep learning benchmarks \citep{semenov2025benchmarking, wen2025fantastic}, 
improving efficiency and robustness compared to their constituent parts. 
They confirm that blending norm-based updates with second-order preconditioning 
yields practical benefits, such as stability at high learning rates \citep{frans2025stable}
and faster convergence with reduced hyperparameter tuning \citep{vyas2024soap}.  

Building on this idea, we introduce two new optimizers: \texttt{MuAdam} and \texttt{MuAdam-SANIA} (Algorithm~\ref{alg:muadam}). 
Both use Muon’s spectral LMO \citep{jordan2024muon} in conjunction with Adam-style \citep{kingma2014adam} or SANIA \citep{abdukhakimov2023sania} 
preconditioners. While \texttt{MuAdam} can be seen as a practical analogue of Adam within a spectral geometry, 
\texttt{MuAdam-SANIA} extends this construction with the scale-invariant SANIA update, 
providing robustness to coordinate-wise rescaling. 
\texttt{MuAdam} can be viewed as a practical optimizer in the same spirit as Adam, 
while \texttt{MuAdam-SANIA} inherits scale-invariance from its preconditioner, 
offering robustness to coordinate-wise rescaling.  

\begin{algorithm}{\texttt{MuAdam} and \texttt{MuAdam-SANIA}}
   \label{alg:muadam}
\begin{algorithmic}[1]
   \State \textbf{Parameters:} step size $\gamma_t$, Adam coefficients $\beta_1,\beta_2 \in (0,1)$, 
   stability constant $\varepsilon > 0$.
   \State \textbf{Initialization:} $M_{-1} = V_{-1} = 0$.

   \For{$t=0,1,2,\dots$}
      \State \textbf{Gradient estimation:} obtain stochastic gradient $G_t = \nabla \mathcal{L}(W_t; \xi_t)$.
      \State \textbf{Moment and precondition updates:} \hfill (Adam \citep{kingma2014adam})
        \[
        M_t = \beta_1 M_{t-1} + (1-\beta_1) G_t, \quad \hat{M}_t = \tfrac{M_t}{1-\beta_1^{t+1}},
        \]
        \[
        V_t = \beta_2 V_{t-1} + (1-\beta_2) (G_t \odot G_t), \quad \hat{V}_t = \tfrac{V_t}{1-\beta_2^{t+1}}.
        \]
        
      \State \textbf{First preconditioning:}
        \[
        N_t = \frac{\hat{M}_t}{\hat{V}_t^{\circ p}+\varepsilon}, 
        \quad \text{where } p = 
        \begin{cases}
           1/4, & \text{for \texttt{MuAdam}}, \\
           1/2, & \text{for \texttt{MuAdam-SANIA}}.
        \end{cases}
        \]

      \State \textbf{Spectral norm based LMO step:} \hfill {(Muon \citep{jordan2024muon})}
        \[
        N_t' = \text{Newton--Schulz}(N_t).
        \]

      \State \textbf{Second preconditioning:}
        \[
        N_t'' = \frac{N_t'}{\hat{V}_t^{\circ p}+\varepsilon}.
        \]

      \State \textbf{Update:} 
        $
        W_{t+1} = W_t - \gamma_t \cdot N_t''.
        $
   \EndFor
\end{algorithmic}
\end{algorithm}
In Algorithm \ref{alg:muadam} we incorporate the momentum term $M_t$, which is standard in 
modern optimization \citep{polyak1964some,kingma2014adam}. 
Furthermore, the appearance of two applications of the element-wise preconditioner $D_t = \hat{V}_t^{\circ p}+\varepsilon$ follows 
directly from Theorem \ref{thm:lmo_general}, which describes how precondition matrices interact with the 
LMO of the base norm. 

Taken together, SOAP \citep{vyas2024soap}, SPlus \citep{frans2025stable}, and proposed \texttt{MuAdam} and \texttt{MuAdam-SANIA} (Algorithm \ref{alg:muadam})
illustrate the potential of systematic combinations of norm-based LMOs with quasi-Newton 
or adaptive preconditioners. Our framework highlights that this design space 
is broad but still under-explored, suggesting a promising direction for future research.

\section{Geometric Properties} \label{ssec:geometric}
As we mentioned earlier, one of the important properties of optimization algorithms is 
\textit{affine} and \textit{scale invariance}~\citep{abdukhakimov2023sania}. Informally, an algorithm is affine invariant if its behavior
is unaffected by a linear reparametrization of the parameters (e.g., changes to the coordinate basis), whereas
scale invariance refers to invariance under coordinate-wise rescaling. 
These properties are desirable because they facilitate the implementation process, ensure that optimization dynamics reflect only geometry of the objective, and can lead to improvement in accuracy \citep{yen2024lora}. In Appendix \ref{app:invariance} we provide a detailed discussion of affine and scale invariance 
for common optimizers.

\paragraph{Invariances for vectors.} 
For the vector-valued losses, affine invariance has been considered only for
transformations of the form $\mathcal{L}_{\text{new}}(w) = \mathcal{L}(A w)$ with a non‑degenerate
matrix $A$ \citep{abdukhakimov2023sania, choudhury2024remove}. 
This covers both full linear reparametrizations and, as a special case of diagonal $A$,
\emph{scale invariance}. 
Formally, we say that an optimization algorithm is scale/affine \emph{invariant} if the iterate sequences of algorithm coincide on $\mathcal{L}$ and its reparametrized version $\mathcal{L}_{\text{new}}$: 
\[
    \mathcal{L}(W_t) \;=\; \mathcal{L}_{\text{new}}(W^{\text{new}}_t) \quad \text{for all iterations $t$},
\] 
where $\{W_t\}$ and $\{W^{\text{new}}_t\}$ denote the iterates produced for $\mathcal{L}$ and $\mathcal{L}_{\text{new}}$, respectively. 

\paragraph{Example.}
Consider a linear model
$
    \mathcal{L}(w) = \big(\langle x, w \rangle - y \big)^2
$
with vector parameters $w$.
If the inputs are linearly transformed, $x \mapsto A^T x$, then the corresponding reparametrized loss is
\[
    \mathcal{L}_{\text{new}}(w) = \big( \langle A^T x, w \rangle - y \big)^2 =  \big( \langle x, A w \rangle - y \big)^2 = \mathcal{L}(A w).
\]
An affine invariant optimization algorithm should then generate identical optimization dynamics under such 
a change of variables. If $A$ is diagonal, this reduces to coordinate‑wise scaling.

\paragraph{Invariances for matrices.}
In this work, we extend these notions for the first time to the case of matrix‑valued
parameters, which are ubiquitous in modern deep learning. We introduce affine
invariance in this setting by defining the reparametrized function as
\[
   \mathcal{L}_{\text{new}}(W) = \mathcal{L}(A_L W A_R),
\]
where $A_L$ and $A_R$ are non‑degenerate matrices. This generalization is consistent
with the classical vector definition but now naturally decomposes into two sides:
left multiplication reflects a change of basis or rescaling of input features, while right multiplication corresponds to transformations of
the output representation. 
Analogously, scale invariance in the matrix case takes the form of
\emph{element‑wise reparametrizations} through the Hadamard product,
\[
   \mathcal{L}_{\text{new}}(W) = \mathcal{L}(A \odot W),
\]
where $A$ is now a positive scaling matrix.
Thus, while in the vector case affine and scale invariance reduce to 
linear and diagonal transformations, in the matrix case they naturally separate into
two distinct but related notions: left/right affine transformations and element‑wise
scaling.

We are going to formalize the conditions under which the LMO step~\eqref{eq:lmo_step} with norms from Definitions \ref{def:norm_L_R} and \ref{def:norm_D}
exhibits affine and scale invariance in the matrix setting. 
In both cases, the result provides a necessary and sufficient characterization, 
directly expressed in terms of the transformation rules for the preconditioners. Result for general norm is presented in Theorem \ref{theorem:lmo_aff} in Appendix \ref{app:proofs}.

\begin{theorem}
\label{theorem:affine_invariance}
Let the LMO step~\eqref{eq:lmo_step} be implemented with the
preconditioned norms (as \Cref{thm:lmo_general}) with a base norm $\norm{\cdot}$ such that the solution is unique. Then the conditions for the invariances are:

\textbf{Affine invariance:}
    For the $(L,R,\norm{\cdot})$‑norm (Definition \ref{def:norm_L_R}), the step is affine invariant 
    with respect to transformations $\mathcal{L}_{\text{new}}(W)=\mathcal{L}(A_L W A_R)$ 
    if and only if the left and right preconditioners $(L^{\mathcal{L}}, R^{\mathcal{L}})$ for the loss $\mathcal{L}$ and $(L^{\mathcal{L}_{\text{new}}}, R^{\mathcal{L}_{\text{new}}})$ for the loss $\mathcal{L}_{\text{new}}$ satisfy
    \[
       L^{\mathcal{L}_{\text{new}}} = L^{\mathcal{L}} A_L, \qquad 
       R^{\mathcal{L}_{\text{new}}} = A_R R^{\mathcal{L}} ,
    \]

\textbf{Scale invariance:}
    For the $D$‑norm (Definition \ref{def:norm_D}), the step is scale invariant with respect to 
    element‑wise rescaling $\mathcal{L}_{\text{new}}(W) = \mathcal{L}(A \odot W)$ 
    if and only if element-wise preconditioners $D^{\mathcal{L}}$ and $D^{\mathcal{L}_{\text{new}}}$ for losses $\mathcal{L}$ and $\mathcal{L}_{\text{new}}$ satisfy
    \[
       D^{\mathcal{L}_{\text{new}}} = A \odot D^{\mathcal{L}}.
    \]
\end{theorem}

The theorem shows that invariance is preserved precisely when the preconditioners 
transform consistently with the underlying reparametrization: 
$L,R$ under affine changes and $D$ under coordinate‑wise scaling. 
This provides a clear and practical criterion: the property is not abstract, 
but follows directly from simple structural relations between preconditioners. 
Using Theorem~\ref{theorem:affine_invariance} we now easily proof the scale invariance of \texttt{MuAdam-SANIA}.

\begin{corollary}
\texttt{MuAdam-SANIA} with $\varepsilon = 0$ is scale invariant.
\begin{proof}
    Under coordinate-wise rescaling $\mathcal{L}(W) \mapsto \mathcal{L}_{\text{new}}(W) = \mathcal{L}(A \odot W)$, 
    the gradients transform as $G_t^{\mathcal{L}_{\text{new}}} = A \odot G_t^{\mathcal{L}}$. 
    By induction, the preconditioner update 
    $D_t^{\mathcal{L}} = [\beta_2 D_{t-1}^{\mathcal{L}} + (1-\beta_2)(G_t^{\mathcal{L}} \odot G_t^{\mathcal{L}})]^{\circ 1/2}$ 
    for the function $\mathcal{L}_{\text{new}}$ transforms as
    \begin{align*}
        D^{\mathcal{L}_{\text{new}}}_t &= [\beta_2 D_{t-1}^{\mathcal{L}_{\text{new}}} + (1-\beta_2)(G_t^{\mathcal{L}_{\text{new}}} \odot G_t^{\mathcal{L}_{\text{new}}})]^{\circ 1/2}
        &= A \odot [\beta_2 D_{t-1}^{\mathcal{L}} + (1-\beta_2)(G_t^{\mathcal{L}} \odot G_t^{\mathcal{L}})]^{\circ 1/2} \\
        &= A \odot D^{\mathcal{L}}_t \quad \text{for all $t$}.
    \end{align*}
    Therefore, $D_t$ for \texttt{MuAdam-SANIA} satisfies the condition of Theorem \ref{theorem:affine_invariance} for scale invariance.
\end{proof}
\end{corollary}

\section{Experiments} \label{ssec:experiments}
\subsection{Scale Invariance under Coordinate-wise Rescaling}
\label{ssec:scale_invariance}

We evaluate scale invariance using the Mushrooms dataset from LIBSVM \citep{Chang2011LIBSVM} with a two-layer MLP. To simulate ill-conditioned feature scales, we construct a scaled variant $\tilde{X} = X \,\mathrm{diag}(e)$ where $e_i = \exp(a_i)$ and $a_i \sim \mathrm{Uniform}[-k,k]$ independently. We compare non-scale-invariant methods (AdamW, Muon) against scale-invariant baselines (Adam-SANIA) and our spectral hybrid (\texttt{MuAdam-SANIA}, Algorithm \ref{alg:muadam} with $p=1/2$).

Hyperparameters are tuned using Optuna on validation splits (see Appendix~\ref{app:scale_exp_hyperparams}). As shown in \Cref{fig:scale_exp}, AdamW and Muon degrade on scaled data with higher training loss and reduced test accuracy. In contrast, Adam-SANIA and \texttt{MuAdam-SANIA} maintain overlapping trajectories between original and scaled settings due to scale invariance. \texttt{MuAdam-SANIA} consistently matches or outperforms Adam-SANIA, demonstrating that combining diagonal preconditioning with a spectral LMO yields tangible gains.

\begin{figure}[h]
    \centering
    \includegraphics[width=0.9\linewidth]{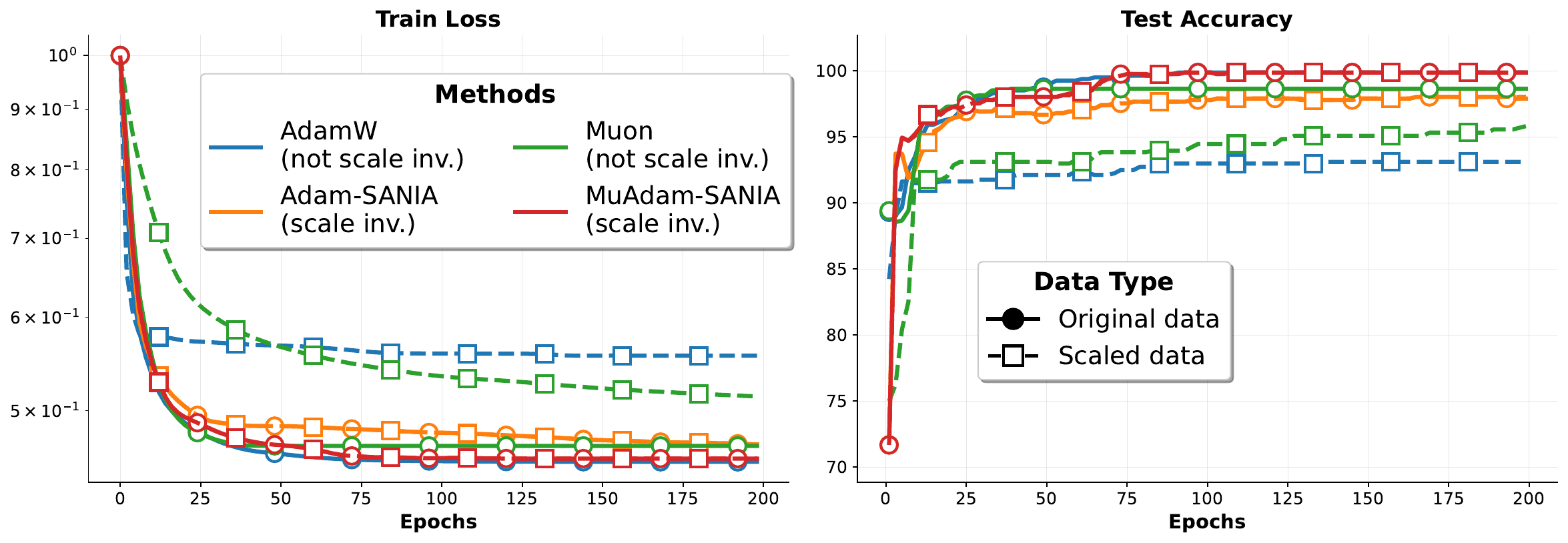}
    \caption{Scale invariance experiment (Mushrooms, LIBSVM) with a two-layer MLP.
    Training loss (left, log-scale) and test accuracy (right) on original vs.\ scaled inputs.}
    \label{fig:scale_exp}
\end{figure}

\subsection{GLUE Benchmark Evaluation}
\label{ssec:glue}

We fine-tune DistilBERT base on GLUE tasks \citep{wang-etal-2018-glue}, comparing \texttt{AdamW}, \texttt{Muon}, and \texttt{MuAdam}. We evaluate both LoRA and full fine-tuning, sweeping four learning rates per dataset and tracking validation metrics. Table~\ref{tab:glue_results_transposed} presents results with columns for datasets and metrics (Matthews correlation for CoLA, accuracy for classification, combined score for STS-B). The \texttt{ALL} column shows task averages.

Full fine-tuning yields higher absolute performance than LoRA, while LoRA remains competitive and parameter-efficient. \texttt{MuAdam} demonstrates competitive performance, often matching or exceeding baselines across tasks, validating the effectiveness of combining spectral geometry with adaptive preconditioning in transformer fine-tuning.

\begin{table}[h!]
\centering
\scriptsize
\setlength{\tabcolsep}{3pt}
\renewcommand{\arraystretch}{1.1}
\captionof{table}{GLUE: datasets are columns with the corresponding metric; \texttt{ALL} is the average over tasks. Best per task in bold.}
\resizebox{0.95\linewidth}{!}{
\begin{tabular}{l lcccccccc|c}
& & \begin{tabular}{@{}c@{}}CoLA\\Matthews\end{tabular} & \begin{tabular}{@{}c@{}}MNLI\\Acc\end{tabular} & \begin{tabular}{@{}c@{}}MRPC\\Acc\end{tabular} & \begin{tabular}{@{}c@{}}QNLI\\Acc\end{tabular} & \begin{tabular}{@{}c@{}}QQP\\Acc\end{tabular} & \begin{tabular}{@{}c@{}}RTE\\Acc\end{tabular} & \begin{tabular}{@{}c@{}}SST-2\\Acc\end{tabular} & \begin{tabular}{@{}c@{}}STS-B\\Comb.\end{tabular} & \begin{tabular}{@{}c@{}}ALL\\Avg\end{tabular} \\
\midrule
\multirow{5}{*}{LoRA} & \texttt{AdaMuon} & 0.5174 & \textbf{0.7836} & 0.8480 & 0.8739 & 0.8584 & \textbf{0.6679} & 0.9037 & 0.8427 & \textbf{0.7869} \\
 & \texttt{AdamW} & 0.4759 & 0.7352 & 0.8309 & 0.8567 & 0.8498 & 0.6390 & 0.8956 & \textbf{0.8473} & 0.7663 \\
 & \texttt{Muon} & 0.4992 & 0.7237 & \textbf{0.8578} & 0.8501 & 0.8322 & 0.6606 & 0.8865 & 0.8392 & 0.7687 \\
 & \texttt{RMSSpectral} & \textbf{0.5327} & 0.7776 & 0.8505 & 0.8682 & \textbf{0.8585} & 0.6534 & 0.8968 & 0.8344 & 0.7840 \\
 & \texttt{RMSSpectral-SANIA} & 0.4990 & 0.7780 & 0.8480 & \textbf{0.8781} & 0.8571 & 0.6209 & \textbf{0.9060} & 0.8422 & 0.7787 \\
\midrule
\multirow{5}{*}{Full} & \texttt{AdaMuon} & 0.4770 & 0.8109 & 0.8260 & 0.8757 & 0.8782 & NaN & NaN & NaN & 0.7736 \\
 & \texttt{AdamW} & \textbf{0.5880} & \textbf{0.8411} & \textbf{0.8578} & \textbf{0.9189} & \textbf{0.9063} & 0.6823 & \textbf{0.9289} & \textbf{0.8932} & \textbf{0.8271} \\
 & \texttt{Muon} & 0.5753 & 0.8068 & 0.8456 & 0.9055 & 0.8958 & \textbf{0.6895} & 0.9197 & 0.8903 & 0.8161 \\
 & \texttt{RMSSpectral} & 0.4800 & 0.8116 & 0.7819 & 0.8794 & 0.8801 & NaN & NaN & NaN & 0.7666 \\
 & \texttt{RMSSpectral-SANIA} & 0.4587 & 0.8122 & 0.8137 & 0.8788 & 0.8782 & NaN & NaN & NaN & 0.7683 \\
\bottomrule
\end{tabular}
}
\label{tab:glue_results_transposed}
\end{table}

\subsection{LLM Fine-Tuning}
\label{ssec:llm}

We fine-tune Qwen2-7B on BoolQ \citep{clark2019boolq}, HellaSwag \citep{zellers2019hellaswag}, and ARC-Challenge \citep{clark2018think} using LoRA, comparing AdamW, Muon, and \texttt{MuAdam} across four learning rates with three random seeds. We report the best accuracy over learning rate sweeps, averaged across seeds with standard deviation error bars.

Figure~\ref{fig:llm_results} shows \texttt{MuAdam} achieves competitive performance, exceeding both AdamW and Muon on BoolQ and HellaSwag while underperforming on ARC-Challenge, demonstrating effectiveness across question-answering and reasoning tasks.

\begin{figure}[t]
\centering
\includegraphics[width=0.8\linewidth]{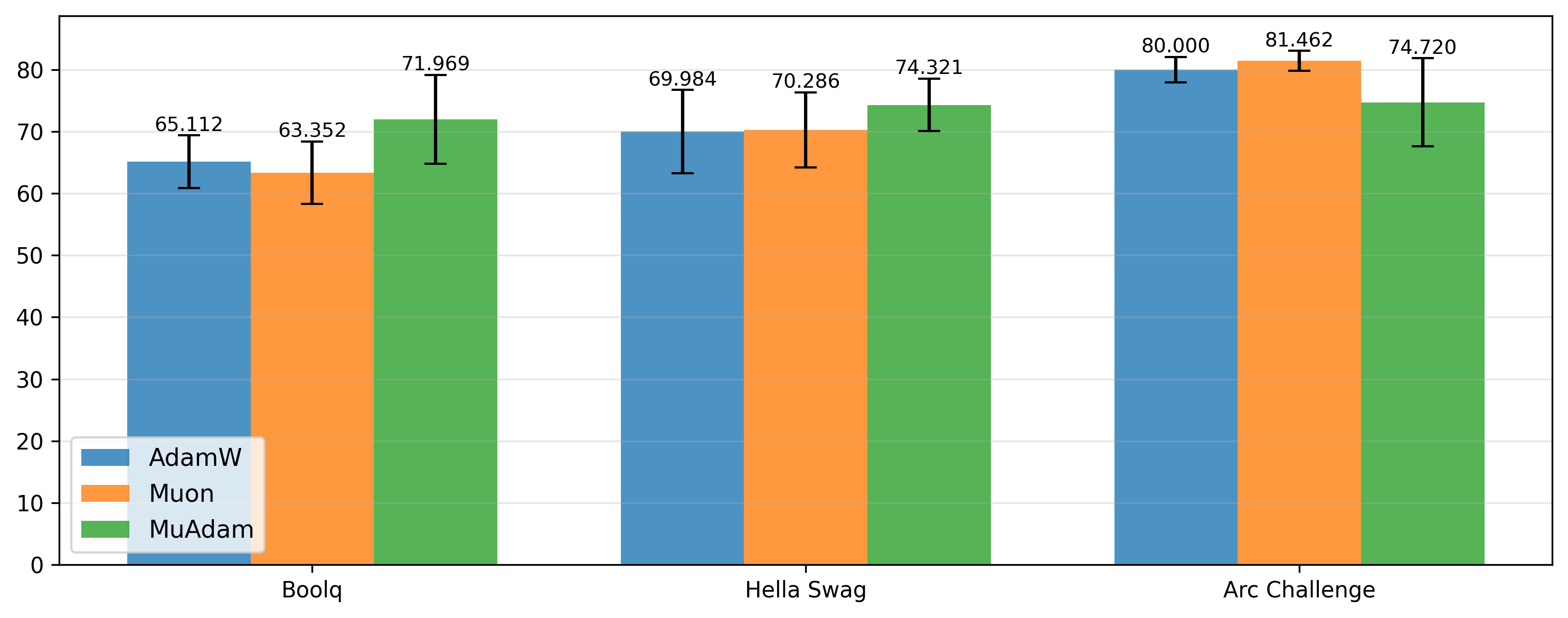}
\caption{LLM fine-tuning results on Qwen2-7B: mean final accuracy with standard deviation across three seeds.}
\label{fig:llm_results}
\end{figure}

\subsection{Character-Level Language Modeling}
\label{ssec:shakespeare}

We evaluate character-level language modeling on the Shakespeare dataset using transformer models with 2, 3, and 4 layers (128 dimensions for 2 layers, 256 for others). Models are trained for 500 epochs with sequence length 256. We perform hyperparameter tuning via random search across batch sizes, learning rates, and dropout values for each optimizer (\texttt{AdamW}, \texttt{Muon}, \texttt{MuAdam}).

Table~\ref{tab:shakespeare_results} shows \texttt{MuAdam} outperforms \texttt{AdamW} on 3 and 4 layer models while slightly underperforming on 2 layers. Both significantly outperform \texttt{Muon} across all configurations, validating that our method successfully combines spectral geometry with adaptive preconditioning.

\begin{table}[h!]
\centering
\captionof{table}{Shakespeare character-level language modeling: best validation accuracy by optimizer and layer count.}
{
\begin{tabular}{l|c|c|c}
& \begin{tabular}{@{}c@{}}2 layers\\Val Acc\end{tabular} & \begin{tabular}{@{}c@{}}3 layers\\Val Acc\end{tabular} & \begin{tabular}{@{}c@{}}4 layers\\Val Acc\end{tabular} \\
\midrule
\texttt{AdamW} & \textbf{0.5506} & 0.5580 & 0.5636 \\

\texttt{Muon} & 0.5382 & 0.5274 & 0.5568 \\

\texttt{MuAdam} & 0.5483 & \textbf{0.5597} & \textbf{0.5664} \\
\bottomrule
\end{tabular}
}
\label{tab:shakespeare_results}
\end{table}

\section{Conclusion}
We introduced a unified framework for optimization with matrix-parameterized models based on preconditioned norms. This abstraction subsumes steepest descent, quasi-Newton, and adaptive methods, and provides the first systematic characterization of affine and scale invariance in the matrix setting. Building upon this foundation, we proposed \texttt{MuAdam} and \texttt{MuAdam-SANIA}, which combine spectral geometry with adaptive preconditioning, achieving strong empirical results across diverse tasks. These results suggest that integrating generalized norms with structured preconditioning offers a rich and still underexplored landscape for the development of next-generation optimization methods.
\end{mainpart}

\begin{appendixpart}

\begin{table}[h!]
    \centering
    \setlength\tabcolsep{3pt} 
    \begin{threeparttable}[th]{
            \captionof{table}{Table of the frequently used notation.}
            \label{tab:table_of_notation}
            \centering 
            \begin{tabular}{c l}
            \Xhline{4\arrayrulewidth}
            \textbf{Notation} & \textbf{Meaning} \\
            \Xhline{2\arrayrulewidth}
            $\mathcal L$ & Objective loss\\
            \hline
            $\otimes$ & Kronecker product\\
            $\odot$ & Hadamard product\\
            $A^{-T}$ & Inverse transpose (of an invertable matrix)\\
            $A^{\circ -1}$ & Elementwise inverse of the matrix A.\\
            \Xhline{2\arrayrulewidth}
            $W_t$ & Iterate sequence\\
            $\Delta W_t$ & Update rule in the interation $t$\\
            $G_t$ & Gradient at iterate $t$\\
            $L_t, R_t, D_t, V_t$ & Left, right, diagonal, elementwise preconditioners\\
            $H^L_t, H^R_t$ & Left and right curvature factors\\
            \Xhline{2\arrayrulewidth}
            $\norm {\cdot}_\alpha, \norm {\cdot}_\beta$ & Vector norms\\
            $\norm {\cdot}_F$ & Frobenius norm\\
            $\norm {\cdot}_{\text{RMS}}$ & RMS norm $=\text{dim}(x)^{\frac 12} \norm x_2$\\
            dim$(x)$ & dimension of vector x\\
            $\norm {\cdot}_{\alpha \to \beta}$ & Matrix norm induced by vector norms $\norm {\cdot}_\alpha, \norm {\cdot}_\beta$\\
            \Xhline{2\arrayrulewidth}
            
            $(L,R)$-norm, $D$-norm & Preconditioned matrix norm\\

            \makecell[c]{$\text{lmo}(\cdot), \text{lmo}_{\norm{\cdot}}(\cdot),$ \\ $\text{lmo}_{L, R, \norm{\cdot}}(\cdot), \text{lmo}_{D,\norm{\cdot}}(\cdot)$} & \makecell[l]{Linear minimization oracles with dependence\\ on the base norm and preconditioners} \\
            \Xhline{4\arrayrulewidth}
            \end{tabular}
        }
    \end{threeparttable}
\end{table}
\section{Additional Discussion on Geometric Invariance} \label{app:invariance}

In this subsection we collect the algebraic details that justify the summary given in Section~\ref{ssec:geometric}.  Throughout we fix an invertible matrix $A\in\R^{d\times d}$ and consider the re--parameterized loss for vectorized parameters
\[
    \Phi(w^A):=\el\bigl(A\,w^{A}\bigr), \qquad  w^{A}:=A^{-1}w,
\]
here we used $\Phi$ instead of $\mathcal{L}_{\text{new}}$ as in Section \ref{ssec:geometric} in terms of convenience, we will do the similar change of notation in the Appendix \ref{app:proofs}.
If an optimizer produces iterates $w_{t}$ for $\el$, we denote by $w^{A}_{t}$ the iterates it produces on $\Phi$.  The method is 
\emph{affine--invariant} iff $w^{A}_{t}=A^{-1}w_{t}$ for all $t$.  A weaker property is obtained when $A$ is restricted to be diagonal with positive entries; this is called \emph{scale invariance}.

\paragraph{Newton’s method \citep{nesterov2018lectures} \;(affine invariant).}
With the Hessian $H_t=\nabla^{2}\el(w_t)$, one step of Newton reads
\[
    w_{t+1}=w_{t}-H_t^{-1}\nabla\el(w_t).
\]
Under the change of variables we have $\nabla\Phi(w_t^{A})=A^{\top}\nabla\el(w_t)$ and $\nabla^{2}\Phi(w_t^{A})=A^{\top}H_tA$.  Hence
\[
    w^{A}_{t+1}=w^{A}_{t}-\bigl(A^{\top}H_tA\bigr)^{-1}A^{\top}\nabla\el(w_t)=A^{-1}\Bigl(w_t-H_t^{-1}\nabla\el(w_t)\Bigr)=A^{-1}w_{t+1},
\]
therefore Newton’s iterates transform equivariantly and the method is fully affine invariant.

\paragraph{Stochastic gradient descent \citep{robbins1951stochastic} \;(not invariant).}
SGD updates are
\[
    w_{t+1}=w_t-\gamma_t\,g_t, \qquad g_t:=\nabla\el(w_t,\xi_t).
\]
After re--parameterisation,
\[
    w^{A}_{t+1}=A^{-1}w_t-\gamma_t A\,g_t \neq A^{-1}w_{t+1},
\]
therefore neither affine nor scale invariance is satisfied.

\paragraph{Adam \citep{kingma2014adam} (with $\varepsilon\!=\!0$) \;(not scale invariant).}
With exponential moving averages $m_t,v_t$ the Adam step is
\[
    w_{t+1}=w_t-\gamma_t\,\frac{m_t}{\sqrt{v_t}}, \quad
    m_t=\beta_1 m_{t-1}+(1-\beta_1)g_t, \quad
    v_t=\beta_2 v_{t-1}+(1-\beta_2)g_t\odot g_t.
\]
Under a diagonal rescaling $A=\mathrm{diag}(a_1,\dots,a_d)$ (coordinate--wise scale change) we have $m^{A}_t=A m_t$ and $v^{A}_t=A^{\!2}v_t$, so
\[
    w^{A}_{t+1}=A^{-1}w_t-\gamma_t\,A\,\frac{m_t}{\sqrt{v_t}}
    \neq A^{-1}w_{t+1}.
\]
The mismatch comes from the square root in the denominator.

\paragraph{SANIA \citep{abdukhakimov2023sania} (scale invariant).}
SANIA removes the square root and normalises by $v_t$ itself:
\[
    w_{t+1}=w_t-\gamma_t\,\frac{m_t}{v_t}.
\]
With the same relations $m^{A}_t=A m_t$ and $v^{A}_t=A^{2} v_t$ we obtain
\[
    w^{A}_{t+1}=A^{-1}w_t-\gamma_t\,A^{-1}\frac{m_t}{v_t}=A^{-1}w_{t+1},
\]
which shows exact scale invariance.  (Full affine invariance is not expected because the pre–conditioner is diagonal.)
\newpage
\section{Missing Proofs}
\label{app:proofs}

\subsection{Proof of Theorem \ref{thm:lmo_general}}

\begin{proof}
By definition of the linear minimization oracle~\eqref{eq:lmo}, we have for a general preconditioned norm
\begin{equation*}
    \lmo_{\mathcal{P}, \norm{\cdot}}(G) 
    = \arg\min_{T : \norm{\mathcal{P}(T)} \leq \rho} \langle G, T \rangle,
\end{equation*}
where $\mathcal{P}$ denotes the preconditioning transform.  
We now proceed case by case.

\medskip\noindent
(i) $(L,R)$‑norm (Definition \ref{def:norm_L_R}).
Here $\mathcal{P}(T) = L T R$.  
Setting $Q = L T R$ (so $T = L^{-1} Q R^{-1}$), and noting that the mapping is bijective since $L,R$ are non‑degenerate, we obtain
\begin{equation*}
\begin{split}
    \lmo_{L,R,\norm{\cdot}}(G)
    &= L^{-1} \cdot 
       \arg\min_{\norm{Q}\leq \rho} \langle G, L^{-1} Q R^{-1}\rangle 
       \cdot R^{-1} \\
    &= L^{-1} \cdot 
       \arg\min_{\norm{Q}\leq \rho} \langle L^{-T} G R^{-T}, Q \rangle 
       \cdot R^{-1} \\
    &= L^{-1} \cdot \lmo_{\norm{\cdot}}(L^{-T} G R^{-T}) \cdot R^{-1}.
\end{split}
\end{equation*}

\medskip\noindent
(ii) $D$‑norm (Definition \ref{def:norm_D}).
Here $\mathcal{P}(T) = D \odot T$.  
Let $Q = D \odot T$, so that $T = D^{\circ -1} \odot Q$.  
Since $D$ has strictly positive entries, this mapping is also bijective. Substituting, we get
\begin{equation*}
\begin{split}
    \lmo_{D,\norm{\cdot}}(G)
    &= D^{\circ -1}\odot 
       \arg\min_{\norm{Q}\leq \rho} \langle G, D^{\circ -1}\odot Q \rangle \\
    &= D^{\circ -1}\odot 
       \arg\min_{\norm{Q}\leq \rho} \langle D^{\circ -1}\odot G, Q \rangle \\
    &= D^{\circ -1}\odot \lmo_{\norm{\cdot}}(D^{\circ -1}\odot G).
\end{split}
\end{equation*}

\medskip
Thus, in both cases the LMO acts not before or after preconditioning, but within a transformed gradient space, yielding the claimed formulas for $(L,R)$‑ and $D$‑norms.
\end{proof}

\subsection{Proof of Theorem \ref{theorem:affine_invariance}}

First we need to prove a more general theorem about affine invariance and arbitrary norm. Also, as in Appendix \ref{app:invariance}, for convenience in this section, we will use a notation $\Phi$ for the changed function, instead of $\mathcal{L}_{\text{new}}$ as we did it in Section \ref{ssec:geometric}.

\begin{theorem}
\label{theorem:lmo_aff}
    Define norms we use for running algorithm \eqref{eq:lmo_step} for functions $\mathcal{L}(W)$ and $\phi(\Theta)$ respectively as $\norm{\cdot}_\mathcal{L}$ and $\norm{\cdot}_\phi$. Then for the step \eqref{eq:lmo_step} to be affine invariant it is necessary that for all matrices $T$ of the proper shape and for all nondegenerate matrices $A_L$ and $A_R$ of the proper shape to satisfy this equation:
    \begin{equation*}
        \norm{T}_\el = \norm{A_L^{-1} T A_R^{-1}}_\phi .
    \end{equation*}
    In order for this condition to become sufficient, we need to require the uniqueness of finding $\arg\min$ in lmo \eqref{eq:lmo}, i.e., for all $G$ of the proper shape it holds that
    \begin{equation*}
        \left| \arg\min_{\norm{T}_\el \leq \rho} \left\{ \langle G, T \rangle \right\} \right| = 1.
    \end{equation*}
    
\end{theorem}
\begin{proof}
    Since $\phi(\Theta) = \el(A_L \Theta A_R)$, for any optimization algorithm to be affine invariant it is necessary and sufficient to it's output to satisfy $\Theta^t = A_L^{-1} W_t A_R^{-1}$. Since steps in the steepest descent \eqref{eq:lmo_step} are of the form $$W^{t+1} = W^t - \lmo_\el(\nabla\el(W_t)) ~~\text{and}~~ \Theta^{t+1} = \Theta^t - \lmo_\phi(\nabla\phi(\Theta_t)),$$
    where $\lmo_\el$ and $\lmo_\phi$ are linear minimization oracles based on the norms $\norm{\cdot}_\el$ and $\norm{\cdot}_\phi$.

    From the mathematical induction and the fact that $\Theta^0 = A_L^{-1} W_0 A_R^{-1}$, condition $\Theta^t = A_L^{-1} W_t A_R^{-1}$ equivalent to: 
    \begin{equation}
    \label{eq:lmo_aff_0}
        \lmo_\phi(\nabla\phi(\Theta_t)) = A_L^{-1} \cdot \lmo_\el(\nabla\el(W_t) \cdot A_R^{-1}.
    \end{equation}
    For the function $\el(W)$ we have 
    \begin{equation}
    \label{eq:lmo_aff_1}
        \lmo_\el(\nabla\el(W_t)) = \arg\min_{\norm{Q}_\el \leq \rho} \left\{ \langle\nabla\el(W_t), Q \rangle \right\}.
    \end{equation}
    For the function $\phi(\Theta)$, if $\Theta^t = A_L^{-1} W_t A_R^{-1}$, we have
    \begin{equation*}
        \nabla \phi(\Theta_t) 
        = 
        \nabla_{\Theta_t} \el(A_L \Theta A_R) 
        =
        A_L^T \nabla_{A_L \Theta A_R} \el(A_L \Theta_t A_R) A_R^T
        =
        A_L^T \nabla\el(W_t) A_R^T.
    \end{equation*}

    Therefore lmo for the function $\phi$ takes form 
    \begin{align*}
        \lmo_{\phi}(\nabla \phi(\Theta_t)) 
        &= \lmo_\phi(A_L^T \nabla\el(W_t) A_R^T)
        =
        \arg\min_{\norm{Q}_\phi \leq \rho} \left\{ \langle A_L^T \nabla\el(W_t) A_R^T, Q \rangle \right\}
        \\&=
        \arg\min_{\norm{Q}_\phi \leq \rho} \left\{ \langle \nabla\el(W_t), A_L Q A_R\rangle \right\}
    \end{align*}
    Since matrix $A_{L, R}$ are nondegenerate, we can make a variable substitution $T = A_L Q A_R$ and obtain that
    \begin{align}
    \label{eq:lmo_aff_2}
        \lmo_{\phi}(\nabla \phi(\Theta_t)) 
        &=
        A_L^{-1} \cdot \arg\min_{T: \norm{A_L^{-1} T A_R^{-1}}_\phi \leq \rho} \left\{ \langle \nabla\el(W_t), T\rangle \right\} \cdot A_R^{-1}
    \end{align}
    Combining equations \eqref{eq:lmo_aff_0}, \eqref{eq:lmo_aff_1} and \eqref{eq:lmo_aff_2}, we can obtain that for the affine invariance it is necessary and sufficient that for all matrixes $G$ of the proper shape, this equality holds (we change $Q$ to $T$ in \eqref{eq:lmo_aff_1} for convenience):
    \begin{equation*}
        \arg\min_{\norm{T}_\el \leq \rho} \left\{ \langle G, T \rangle \right\}
        =
        \arg\min_{T: \norm{A_L^{-1} T A_R^{-1}}_\phi \leq \rho} \left\{ \langle G, T\rangle \right\}.
    \end{equation*}
    Since minimization functions are the same for both $\arg\min$, the necessary condition of affine invariance is that for all matrices of the proper shape:
    \begin{equation}
    \label{eq:lmo_aff_3}
        \norm{T}_\el = \norm{A_L^{-1} T A_R^{-1}}_\phi
    \end{equation}
    However equation \eqref{eq:lmo_aff_3} is not sufficient, because we need also to require the uniqueness of this $\arg\min$.
\end{proof}

We now ready to proof Theorem \ref{theorem:affine_invariance}.

\begin{proof}[Proof Theorem \ref{theorem:affine_invariance}]
    The proof of Theorem \ref{theorem:affine_invariance} consists of a straightforward application of Theorem \ref{theorem:lmo_aff}. For all matrices $T \in \R^{m \times n}$ the following equality should hold:
    \begin{equation*}
        \norm{L_\el T R_\el} = \norm{L_\phi A_L^{-1} T A_R^{-1} R_\phi}.
    \end{equation*}
    Therefore the necessary condition on the affine invariance is 
    \begin{equation*}
        L_{\el} = L_\phi A_L^{-1} ~~\text{and}~~ R_\el = A_R^{-1} R_\phi.
    \end{equation*}
    In order for this condition to become sufficient, we need to require the uniqueness of finding $\arg\min$ in lmo with the norm $\norm{\cdot}_{L, R, \norm{\cdot}}$. Since matrices $L$ and $R$ for $\el(W)$ and $\phi(\Theta)$ are non-degenerative, the uniqueness deepens only on the norm $\norm{\cdot}$, i.e.,  
    \begin{equation*}
        \left| \arg\min_{\norm{T} \leq \rho} \left\{ \langle G, T \rangle \right\} \right| = 1.
    \end{equation*}

    The scale invariance case is proven in a similar way as for Theorem \ref{thm:lmo_general}.
\end{proof}
\newpage

\subsection{Polar Express on GLUE Benchmark}
\label{ssec:glue}

Polar Express \citep{amsel2025polar} is a stable, iteration-efficient alternative to Newton--Schulz for computing the polar decomposition of a matrix. Replacing Newton--Schulz with Polar Express improves MNLI/QQP and stays competitive elsewhere on GLUE tasks.

\begin{table}[h!]
\centering
\scriptsize
\setlength{\tabcolsep}{3pt}
\renewcommand{\arraystretch}{1.1}
\captionof{table}{GLUE (LoRA): datasets are columns with the corresponding metric; \texttt{ALL} is the average over tasks. Best per task in bold.}
\resizebox{0.95\linewidth}{!}{
\begin{tabular}{l lcccccccc|c}
& & \begin{tabular}{@{}c@{}}CoLA\\Matthews\end{tabular} & \begin{tabular}{@{}c@{}}MNLI\\Acc\end{tabular} & \begin{tabular}{@{}c@{}}MRPC\\Acc\end{tabular} & \begin{tabular}{@{}c@{}}QNLI\\Acc\end{tabular} & \begin{tabular}{@{}c@{}}QQP\\Acc\end{tabular} & \begin{tabular}{@{}c@{}}RTE\\Acc\end{tabular} & \begin{tabular}{@{}c@{}}SST-2\\Acc\end{tabular} & \begin{tabular}{@{}c@{}}STS-B\\Comb.\end{tabular} & \begin{tabular}{@{}c@{}}ALL\\Avg\end{tabular} \\
\midrule
\multirow{3}{*}{LoRA}
& \texttt{AdamW}   & \textbf{0.5401} & 0.7986 & \textbf{0.8578} & \textbf{0.8788} & 0.8756 & \textbf{0.6751} & \textbf{0.9094} & \textbf{0.8634} & \textbf{0.7999} \\
& \texttt{Muon}    & 0.4787 & 0.7642 & 0.8382 & 0.8589 & 0.8739 & 0.6282 & 0.9037 & 0.8516 & 0.7747 \\
& \texttt{MuAdam}  & 0.4939 & \textbf{0.7987} & 0.8407 & 0.8768 & \textbf{0.8966} & 0.6137 & 0.9037 & 0.8508 & 0.7844 \\
\bottomrule
\end{tabular}
}
\label{tab:glue_results_transposed}
\end{table}

\newpage

\section{Scale Invariance Setup and Hyperparameters}
\label{app:scale_exp_hyperparams}

To ensure a fair comparison across optimizers and input scalings, we perform hyperparameter tuning separately for each method and for both the original and scaled tasks using Optuna on a held-out validation split (see \Cref{ssec:scale_invariance}). Tuned values are selected to maximize validation accuracy, and final results are reported on the test split with the chosen configuration. Below we summarize the key training and tuning settings used in our experiments.
\begin{table}[h!]
\centering
\captionof{table}{Summary of training and tuning hyperparameters.}
\label{tab:hparams}
\renewcommand{\arraystretch}{1.2}
\begin{tabular}{l l}
\toprule
\textbf{Parameter} & \textbf{Value} \\
\midrule
Learning rate (tuned by Optuna) & $\mathrm{LogUniform}[1e\text{-}6, 5e0]$ \\
Weight decay (tuned by Optuna) & $\mathrm{LogUniform}[1e\text{-}6, 1e\text{-}2]$ \\
Batch size & \verb|len(train_dataloader)| (full-batch) \\
Training epochs & 200 \\
Tuning epochs & 10 \\
Optuna runs per setting & 40 \\
Hidden dimension (MLP) & 100 \\
Bias terms & Disabled \\
Weight initialization & Input layer: zeros; Output layer: uniform \\
Numerical epsilon & $1\mathrm{e}{-40}$ \\
DType & float64 \\
Scaling bound for data & $k = 10$ \\
Seeds & $\{18, 52, 812\}$ \\
Optimizers compared & AdamW, Muon, Adam-SANIA, \texttt{MuAdam-SANIA} \\
\bottomrule
\end{tabular}
\end{table}

Notes: 
The diagonal preconditioner $D_t$ includes an additive $\varepsilon$, which normally takes values 
around $10^{-8}$ \citep{kingma2014adam}, however this value is big enough to break down scale invariance. 
To minimize this effect we use $\varepsilon=10^{-40}$, therefore its contribution is negligible while still 
ensuring numerical safety. In addition, we employ float64 precision, since the use of such a tiny 
$\varepsilon$ further increases the need for high numerical accuracy, and scale invariance requires well‑defined updates 
even under extremely large or small entries of the scaling matrix $A$.

Also, the input layer of the MLP is initialized with zeroes, which is consistent with the 
scale‑invariance assumption $W_0^{\text{new}} = A^{-1} W_0 = 0$. At the same time, the final output 
layer is initialized with a standard uniform distribution, since setting it to zero would eliminate the 
signal necessary for learning and prevent the network from training.

\newpage
\section{GLUE Fine-Tuning Training Setup and Hyperparameters}
\begin{table}[h!]
\centering
\captionof{table}{GLUE fine-tuning: training setup and sweep.}
\label{tab:glue_hparams}
\renewcommand{\arraystretch}{1.2}
\begin{tabular}{l l}
\toprule
\textbf{Parameter} & \textbf{Value} \\
\midrule
Backbone & \texttt{distilbert/distilbert-base-uncased} \\
Fine-tuning & Full FT or LoRA ($r{=}4$, $\alpha{=}32$, dropout $0.05$) \\
Batch size & 16 \\
Grad. accumulation & 2 \\
DType & bfloat16 \\
LR scheduler & linear, warmup ratio $0.1$ \\
Max train steps & 10000 \\
Eval/save & eval each epoch; no checkpoint saving \\
Sweep LRs & $\{2\!\times\!10^{-4},\,10^{-4},\,5\!\times\!10^{-5},\,3\!\times\!10^{-5}\}$ \\
Optimizers compared & AdamW, Muon, \texttt{MuAdam} \\
\bottomrule
\end{tabular}
\end{table}

Notes: results are picked as the best validation score observed across all evaluation steps within each run, then the best over the LR sweep per optimizer.
\newpage
\section{LLM Fine-Tuning Setup and Hyperparameters}
\label{app:llm_hparams}

We fine-tune Qwen2-7B on three reasoning datasets using LoRA with a grid search over learning rates. Results are averaged over multiple seeds for statistical robustness.

\begin{table}[h!]
\centering
\captionof{table}{LLM fine-tuning: training setup and sweep.}
\label{tab:llm_hparams}
\renewcommand{\arraystretch}{1.2}
\begin{tabular}{l l}
\toprule
\textbf{Parameter} & \textbf{Value} \\
\midrule
Backbone & \texttt{Qwen/Qwen2-7B} \\
Fine-tuning & LoRA ($r{=}16$, $\alpha{=}32$, dropout $0.05$) \\
Batch size & 1 \\
Grad. accumulation & 4 \\
Quantization & 4-bit \\
DType & bfloat16 \\
LR scheduler & linear, warmup ratio $0.1$ \\
Max train steps & 1000 \\
Max seq length & 512 \\
Datasets & BoolQ, HellaSwag, ARC-Challenge \\
Sweep LRs & $\{2\!\times\!10^{-4},\,10^{-4},\,5\!\times\!10^{-5},\,3\!\times\!10^{-5}\}$ \\
Seeds & $\{42, 123, 456\}$ \\
Optimizers compared & AdamW, Muon, \texttt{MuAdam} \\
\bottomrule
\end{tabular}
\end{table}

Notes: final accuracy is selected as the best value across all evaluation steps within each run, then the best over the LR sweep per optimizer. Results are averaged across three seeds with standard deviation reported.

\newpage
\section{Character-Level Language Modeling Setup and Hyperparameters}
\label{app:shakespeare_hparams}

We evaluate optimizers on character-level language modeling using the Shakespeare dataset with transformer models of varying architecture complexity. Hyperparameters are tuned using random search across multiple configurations to ensure fair comparison between optimizers.

\begin{table}[h!]
\centering
\captionof{table}{Shakespeare character-level language modeling: training setup and sweep.}
\label{tab:shakespeare_hparams}
\renewcommand{\arraystretch}{1.2}
\begin{tabular}{l l}
\toprule
\textbf{Parameter} & \textbf{Value} \\
\midrule
Model & Transformer (base config) \\
Dataset & \texttt{shakespeare-char} \\
Model layers & 2, 3, 4 \\
Attention heads & 4 \\
Embedding dim & 128 (2 layers), 256 (3-4 layers) \\
Vocab size & 96 \\
Batch size & Random search: \{32, 128, 256\} \\
Sequence length & 256 \\
Grad clip & 0.5 \\
Weight decay & 0.1 \\
LR scheduler & Cosine \\
Dropout & Random search: \{0.05, 0.1, 0.15, 0.25\} \\
Beta1, Beta2 & 0.9, 0.999 \\
LR sweep & $\{10^{-6}, 5\cdot 10^{-6}, 10^{-5}, 5\cdot 10^{-5}, 10^{-4}, 5\cdot 10^{-4}, 10^{-3}, 5\cdot 10^{-3}, 10^{-2}, 3 \cdot 10^{-2}\}$ \\
Optimizers compared & AdamW, Muon, \texttt{MuAdam} \\
\bottomrule
\end{tabular}
\end{table}
\end{appendixpart}

\begin{thebibliography}{70}
\providecommand{\natexlab}[1]{#1}
\providecommand{\url}[1]{\texttt{#1}}
\expandafter\ifx\csname urlstyle\endcsname\relax
  \providecommand{\doi}[1]{doi: #1}\else
  \providecommand{\doi}{doi: \begingroup \urlstyle{rm}\Url}\fi

\bibitem[Abdukhakimov et~al.(2023)Abdukhakimov, Xiang, Kamzolov, Gower, and Tak{\'a}{\v{c}}]{abdukhakimov2023sania}
Farshed Abdukhakimov, Chulu Xiang, Dmitry Kamzolov, Robert Gower, and Martin Tak{\'a}{\v{c}}.
\newblock Sania: Polyak-type optimization framework leads to scale invariant stochastic algorithms.
\newblock \emph{arXiv preprint arXiv:2312.17369}, 2023.

\bibitem[Amsel et~al.(2025)Amsel, Persson, Musco, and Gower]{amsel2025polar}
Noah Amsel, David Persson, Christopher Musco, and Robert Gower.
\newblock The polar express: Optimal matrix sign methods and their application to the muon algorithm.
\newblock \emph{ArXiv}, abs/2505.16932, 2025.
\newblock URL \url{https://api.semanticscholar.org/CorpusID:278789329}.

\bibitem[Autonne(1902)]{autonne1902groupes}
L{\'e}on Autonne.
\newblock Sur les groupes lin{\'e}aires, r{\'e}els et orthogonaux.
\newblock \emph{Bulletin de la soci{\'e}t{\'e} math{\'e}matique de France}, 30:\penalty0 121--134, 1902.

\bibitem[Bernstein and Newhouse(2024)]{bernstein2024old}
Jeremy Bernstein and Laker Newhouse.
\newblock Old optimizer, new norm: An anthology.
\newblock \emph{arXiv preprint arXiv:2409.20325}, 2024.

\bibitem[Bernstein and Newhouse(2025)]{bernstein2024modular}
Jeremy Bernstein and Laker Newhouse.
\newblock Modular duality in deep learning.
\newblock In \emph{International Conference on Machine Learning}, volume 267 of \emph{Proceedings of Machine Learning Research}, pages 3920--3930. PMLR, 2025.

\bibitem[Bernstein et~al.(2018)Bernstein, Wang, Azizzadenesheli, and Anandkumar]{bernstein2018signsgd}
Jeremy Bernstein, Yu-Xiang Wang, Kamyar Azizzadenesheli, and Animashree Anandkumar.
\newblock {signSGD}: Compressed optimisation for non-convex problems.
\newblock In \emph{International Conference on Machine Learning}, pages 560--569. PMLR, 2018.

\bibitem[Bj{\"o}rck and Bowie(1971)]{bjorck1971iterative}
{\AA}ke Bj{\"o}rck and Clazett Bowie.
\newblock An iterative algorithm for computing the best estimate of an orthogonal matrix.
\newblock \emph{SIAM Journal on Numerical Analysis}, 8\penalty0 (2):\penalty0 358--364, 1971.

\bibitem[Bottou(2010)]{bottou2010large}
L{\'e}on Bottou.
\newblock Large-scale machine learning with stochastic gradient descent.
\newblock In \emph{Proceedings of COMPSTAT'2010: 19th International Conference on Computational StatisticsParis France, August 22-27, 2010 Keynote, Invited and Contributed Papers}, pages 177--186. Springer, 2010.

\bibitem[Brown et~al.(2020)Brown, Mann, Ryder, Subbiah, Kaplan, Dhariwal, Neelakantan, Shyam, Sastry, Askell, et~al.]{brown2020language}
Tom Brown, Benjamin Mann, Nick Ryder, Melanie Subbiah, Jared Kaplan, Prafulla Dhariwal, Arvind Neelakantan, Pranav Shyam, Girish Sastry, Amanda Askell, et~al.
\newblock Language models are few-shot learners.
\newblock \emph{Advances in Neural Information Processing Systems}, 33:\penalty0 1877--1901, 2020.

\bibitem[Broyden(1970)]{broyden1970convergence}
Charles Broyden.
\newblock The convergence of a class of double-rank minimization algorithms 1. general considerations.
\newblock \emph{IMA Journal of Applied Mathematics}, 6\penalty0 (1):\penalty0 76--90, 1970.

\bibitem[Broyden(1967)]{broyden1967quasi}
Charles~G Broyden.
\newblock Quasi-newton methods and their application to function minimisation.
\newblock \emph{Mathematics of Computation}, 21\penalty0 (99):\penalty0 368--381, 1967.

\bibitem[Carlson et~al.(2015{\natexlab{a}})Carlson, Cevher, and Carin]{carlson2015stochastic}
David Carlson, Volkan Cevher, and Lawrence Carin.
\newblock {Stochastic Spectral Descent for Restricted Boltzmann Machines}.
\newblock In Guy Lebanon and S.~V.~N. Vishwanathan, editors, \emph{Proceedings of the Eighteenth International Conference on Artificial Intelligence and Statistics}, volume~38 of \emph{Proceedings of Machine Learning Research}, pages 111--119, San Diego, California, USA, 09--12 May 2015{\natexlab{a}}. PMLR.
\newblock URL \url{https://proceedings.mlr.press/v38/carlson15.html}.

\bibitem[Carlson et~al.(2015{\natexlab{b}})Carlson, Collins, Hsieh, Carin, and Cevher]{carlson2015preconditioned}
David Carlson, Edo Collins, Ya-Ping Hsieh, Lawrence Carin, and Volkan Cevher.
\newblock Preconditioned spectral descent for deep learning.
\newblock \emph{Advances in neural information processing systems}, 28, 2015{\natexlab{b}}.

\bibitem[Carlson et~al.(2015{\natexlab{c}})Carlson, Hsieh, Collins, Carin, and Cevher]{carlson2015stochasticgraphical}
David Carlson, Ya-Ping Hsieh, Edo Collins, Lawrence Carin, and Volkan Cevher.
\newblock Stochastic spectral descent for discrete graphical models.
\newblock \emph{IEEE Journal of Selected Topics in Signal Processing}, 10\penalty0 (2):\penalty0 296--311, 2015{\natexlab{c}}.

\bibitem[Chang and Lin(2011)]{Chang2011LIBSVM}
Chih-Chung Chang and Chih-Jen Lin.
\newblock {LIBSVM: A Library for Support Vector Machines}.
\newblock \emph{ACM Transactions on Intelligent Systems and Technology (TIST)}, 2\penalty0 (3):\penalty0 1--27, 2011.

\bibitem[Chen et~al.(2025)Chen, Dong, Wei, Huang, Zhang, Su, and Zhu]{chen2025understanding}
Huanran Chen, Yinpeng Dong, Zeming Wei, Yao Huang, Yichi Zhang, Hang Su, and Jun Zhu.
\newblock Unveiling the basin-like loss landscape in large language models.
\newblock \emph{arXiv preprint arXiv:2505.17646}, 2025.

\bibitem[Choromanska et~al.(2015)Choromanska, Henaff, Mathieu, Arous, and LeCun]{choromanska2015loss}
Anna Choromanska, Mikael Henaff, Michael Mathieu, G{\'e}rard~Ben Arous, and Yann LeCun.
\newblock The loss surfaces of multilayer networks.
\newblock In \emph{Artificial intelligence and statistics}, pages 192--204. PMLR, 2015.

\bibitem[Choudhury et~al.(2024)Choudhury, Tupitsa, Loizou, Horv{\'a}th, Takac, and Gorbunov]{choudhury2024remove}
Sayantan Choudhury, Nazarii Tupitsa, Nicolas Loizou, Samuel Horv{\'a}th, Martin Takac, and Eduard Gorbunov.
\newblock Remove that square root: A new efficient scale-invariant version of adagrad.
\newblock \emph{Advances in Neural Information Processing Systems}, 37:\penalty0 47400--47431, 2024.

\bibitem[Clark et~al.(2019)Clark, Lee, Chang, Kwiatkowski, Collins, and Toutanova]{clark2019boolq}
Christopher Clark, Kenton Lee, Ming-Wei Chang, Tom Kwiatkowski, Michael Collins, and Kristina Toutanova.
\newblock Boolq: Exploring the surprising difficulty of natural yes/no questions, 2019.

\bibitem[Clark et~al.(2018)Clark, Cowhey, Etzioni, Khot, Sabharwal, Schoenick, and Tafjord]{clark2018think}
Peter Clark, Isaac Cowhey, Oren Etzioni, Tushar Khot, Ashish Sabharwal, Carissa Schoenick, and Oyvind Tafjord.
\newblock Think you have solved question answering? try arc, the ai2 reasoning challenge, 2018.

\bibitem[Crawshaw et~al.(2025)Crawshaw, Modi, Liu, and Gower]{crawshaw2025exploration}
Michael Crawshaw, Chirag Modi, Mingrui Liu, and Robert Gower.
\newblock An exploration of non-euclidean gradient descent: Muon and its many variants.
\newblock \emph{arXiv preprint arXiv:2510.09827}, 2025.

\bibitem[d'Aspremont et~al.(2018)d'Aspremont, Guzman, and Jaggi]{d2018optimal}
Alexandre d'Aspremont, Cristobal Guzman, and Martin Jaggi.
\newblock Optimal affine-invariant smooth minimization algorithms.
\newblock \emph{SIAM Journal on Optimization}, 28\penalty0 (3):\penalty0 2384--2405, 2018.

\bibitem[Davidon(1959)]{davidon1959variable}
William Davidon.
\newblock Variable metric method for minimization.
\newblock Technical report, Argonne National Lab., Lemont, Ill., 1959.

\bibitem[Dean et~al.(2012)Dean, Corrado, Monga, Chen, Devin, Mao, Ranzato, Senior, Tucker, Yang, et~al.]{dean2012large}
Jeffrey Dean, Greg Corrado, Rajat Monga, Kai Chen, Matthieu Devin, Mark Mao, Marc'aurelio Ranzato, Andrew Senior, Paul Tucker, Ke~Yang, et~al.
\newblock Large scale distributed deep networks.
\newblock \emph{Advances in neural information processing systems}, 25, 2012.

\bibitem[Defazio and Jelassi(2022)]{defazio2022momentumized}
Aaron Defazio and Samy Jelassi.
\newblock Adaptivity without compromise: A momentumized, adaptive, dual averaged gradient method for stochastic optimization.
\newblock \emph{Journal of Machine Learning Research}, 23\penalty0 (144):\penalty0 1--34, 2022.

\bibitem[Duchi et~al.(2011)Duchi, Hazan, and Singer]{duchi2011adaptive}
John Duchi, Elad Hazan, and Yoram Singer.
\newblock Adaptive subgradient methods for online learning and stochastic optimization.
\newblock \emph{Journal of machine learning research}, 12\penalty0 (7), 2011.

\bibitem[Fletcher(1970)]{fletcher1970new}
Roger Fletcher.
\newblock A new approach to variable metric algorithms.
\newblock \emph{The Computer Journal}, 13\penalty0 (3):\penalty0 317--322, 1970.

\bibitem[Frans et~al.(2025)Frans, Levine, and Abbeel]{frans2025stable}
Kevin Frans, Sergey Levine, and Pieter Abbeel.
\newblock A stable whitening optimizer for efficient neural network training.
\newblock In \emph{Advances in Neural Information Processing Systems}, 2025.
\newblock arXiv:2506.07254.

\bibitem[Gao et~al.(2025)Gao, Chu, Ye, and Udell]{gao2024gradient}
Wenzhi Gao, Ya-Chi Chu, Yinyu Ye, and Madeleine Udell.
\newblock Gradient methods with online scaling.
\newblock In \emph{Conference on Learning Theory}, volume 291 of \emph{Proceedings of Machine Learning Research}, pages 2192--2226. PMLR, 2025.

\bibitem[Goldfarb(1970)]{goldfarb1970family}
Donald Goldfarb.
\newblock A family of variable-metric methods derived by variational means.
\newblock \emph{Mathematics of Computation}, 24\penalty0 (109):\penalty0 23--26, 1970.

\bibitem[Goldfarb et~al.(2020)Goldfarb, Ren, and Bahamou]{goldfarb2020practical}
Donald Goldfarb, Yi~Ren, and Achraf Bahamou.
\newblock Practical quasi-newton methods for training deep neural networks.
\newblock \emph{Advances in Neural Information Processing Systems}, 33:\penalty0 2386--2396, 2020.

\bibitem[Goodfellow et~al.(2014)Goodfellow, Pouget-Abadie, Mirza, Xu, Warde-Farley, Ozair, Courville, and Bengio]{goodfellow2014generative}
Ian Goodfellow, Jean Pouget-Abadie, Mehdi Mirza, Bing Xu, David Warde-Farley, Sherjil Ozair, Aaron Courville, and Yoshua Bengio.
\newblock Generative adversarial nets.
\newblock \emph{Advances in Neural Information Processing Systems}, 27, 2014.

\bibitem[Goodfellow et~al.(2016)Goodfellow, Bengio, Courville, and Bengio]{goodfellow2016deep}
Ian Goodfellow, Yoshua Bengio, Aaron Courville, and Yoshua Bengio.
\newblock \emph{Deep learning}, volume~1.
\newblock MIT press Cambridge, 2016.

\bibitem[Gupta et~al.(2018)Gupta, Koren, and Singer]{gupta2018shampoo}
Vineet Gupta, Tomer Koren, and Yoram Singer.
\newblock Shampoo: Preconditioned stochastic tensor optimization.
\newblock In \emph{International Conference on Machine Learning}, pages 1842--1850. PMLR, 2018.

\bibitem[Hazan et~al.(2015)Hazan, Levy, and Shalev-Shwartz]{hazan2015beyond}
Elad Hazan, Kfir Levy, and Shai Shalev-Shwartz.
\newblock Beyond convexity: Stochastic quasi-convex optimization.
\newblock \emph{Advances in Neural Information Processing Systems}, 28, 2015.

\bibitem[Hern{\'a}ndez-Cano et~al.(2025)Hern{\'a}ndez-Cano, H{\"a}gele, Huang, Romanou, Solergibert, Pasztor, Messmer, Garbaya, {\v{D}}urech, Hakimi, et~al.]{hernandez2025apertus}
Alejandro Hern{\'a}ndez-Cano, Alexander H{\"a}gele, Allen~Hao Huang, Angelika Romanou, Antoni-Joan Solergibert, Barna Pasztor, Bettina Messmer, Dhia Garbaya, Eduard~Frank {\v{D}}urech, Ido Hakimi, et~al.
\newblock Apertus: Democratizing open and compliant llms for global language environments.
\newblock \emph{arXiv preprint arXiv:2509.14233}, 2025.

\bibitem[Higham(1986)]{higham1986computing}
Nicholas~J Higham.
\newblock Computing the polar decomposition—with applications.
\newblock \emph{SIAM Journal on Scientific and Statistical Computing}, 7\penalty0 (4):\penalty0 1160--1174, 1986.

\bibitem[Jordan et~al.(2024)Jordan, Jin, Boza, Jiacheng, Cecista, Newhouse, and Bernstein]{jordan2024muon}
Keller Jordan, Yuchen Jin, Vlado Boza, You Jiacheng, Franz Cecista, Laker Newhouse, and Jeremy Bernstein.
\newblock Muon: An optimizer for hidden layers in neural networks, 2024.
\newblock URL \url{https://kellerjordan.github.io/posts/muon/}.

\bibitem[Kingma and Ba(2014)]{kingma2014adam}
Diederik Kingma and Jimmy Ba.
\newblock Adam: A method for stochastic optimization.
\newblock \emph{arXiv preprint arXiv:1412.6980}, 2014.

\bibitem[Kovarik(1970)]{kovarik1970some}
Zdislav Kovarik.
\newblock Some iterative methods for improving orthonormality.
\newblock \emph{SIAM Journal on Numerical Analysis}, 7\penalty0 (3):\penalty0 386--389, 1970.

\bibitem[Li(2024)]{li2024quasi}
Jiongcheng Li.
\newblock Quasi-newton method of optimization is proved to be a steepest descent method under the ellipsoid norm.
\newblock \emph{arXiv preprint arXiv:2411.11286}, 2024.

\bibitem[Lin et~al.(2025)Lin, Lowe, Dangel, Eschenhagen, Xu, and Grosse]{lin2025understanding}
Wu~Lin, Scott~C. Lowe, Felix Dangel, Runa Eschenhagen, Zikun Xu, and Roger~B. Grosse.
\newblock Understanding and improving shampoo and {SOAP} via {Kullback-Leibler} minimization.
\newblock \emph{arXiv preprint arXiv:2509.03378}, 2025.

\bibitem[Liu and Nocedal(1989)]{liu1989limited}
Dong~C. Liu and Jorge Nocedal.
\newblock On the limited memory bfgs method for large scale optimization.
\newblock \emph{Mathematical Programming}, 45\penalty0 (1):\penalty0 503--528, 1989.

\bibitem[Martens and Grosse(2015)]{martens2015optimizing}
James Martens and Roger Grosse.
\newblock Optimizing neural networks with kronecker-factored approximate curvature.
\newblock In \emph{International Conference on Machine Learning}, pages 2408--2417. PMLR, 2015.

\bibitem[Nakatsukasa and Higham(2012)]{nakatsukasa2012backward}
Yuji Nakatsukasa and Nicholas Higham.
\newblock Backward stability of iterations for computing the polar decomposition.
\newblock \emph{SIAM Journal on Matrix Analysis and Applications}, 33\penalty0 (2):\penalty0 460--479, 2012.

\bibitem[Nakatsukasa et~al.(2010)Nakatsukasa, Bai, and Gygi]{nakatsukasa2010optimizing}
Yuji Nakatsukasa, Zhaojun Bai, and Fran{\c{c}}ois Gygi.
\newblock Optimizing halley's iteration for computing the matrix polar decomposition.
\newblock \emph{SIAM Journal on Matrix Analysis and Applications}, 31\penalty0 (5):\penalty0 2700--2720, 2010.

\bibitem[Nesterov and Nemirovski(1994)]{nesterov1994interior}
Yurii Nesterov and Arkadi Nemirovski.
\newblock \emph{{Interior-Point Polynomial Algorithms in Convex Programming}}.
\newblock SIAM, 1994.

\bibitem[Nesterov et~al.(2018)]{nesterov2018lectures}
Yurii Nesterov et~al.
\newblock \emph{Lectures on convex optimization}, volume 137.
\newblock Springer, 2018.

\bibitem[Pethick et~al.(2025)Pethick, Xie, Antonakopoulos, Zhu, Silveti-Falls, and Cevher]{pethick2025training}
Thomas Pethick, Wanyun Xie, Kimon Antonakopoulos, Zhenyu Zhu, Antonio Silveti-Falls, and Volkan Cevher.
\newblock Training deep learning models with norm-constrained lmos.
\newblock \emph{arXiv preprint arXiv:2502.07529}, 2025.

\bibitem[Polyak(1964)]{polyak1964some}
Boris~T Polyak.
\newblock Some methods of speeding up the convergence of iteration methods.
\newblock \emph{Ussr computational mathematics and mathematical physics}, 4\penalty0 (5):\penalty0 1--17, 1964.

\bibitem[Riabinin et~al.(2025)Riabinin, Shulgin, Gruntkowska, and Richt{\'a}rik]{riabinin2025gluon}
Artem Riabinin, Egor Shulgin, Kaja Gruntkowska, and Peter Richt{\'a}rik.
\newblock Gluon: Making muon \& scion great again! (bridging theory and practice of lmo-based optimizers for llms).
\newblock \emph{arXiv preprint arXiv:2505.13416}, 2025.

\bibitem[Robbins and Monro(1951)]{robbins1951stochastic}
Herbert Robbins and Sutton Monro.
\newblock A stochastic approximation method.
\newblock \emph{The annals of mathematical statistics}, pages 400--407, 1951.

\bibitem[Rombach et~al.(2022)Rombach, Blattmann, Lorenz, Esser, and Ommer]{rombach2022high}
Robin Rombach, Andreas Blattmann, Dominik Lorenz, Patrick Esser, and Bj{\"o}rn Ommer.
\newblock High-resolution image synthesis with latent diffusion models.
\newblock In \emph{Proceedings of the IEEE/CVF conference on computer vision and pattern recognition}, pages 10684--10695, 2022.

\bibitem[Rumelhart et~al.(1986)Rumelhart, Hinton, and Williams]{rumelhart1986learning}
David Rumelhart, Geoffrey Hinton, and Ronald Williams.
\newblock Learning representations by back-propagating errors.
\newblock \emph{Nature}, 323\penalty0 (6088):\penalty0 533--536, 1986.

\bibitem[Semenov et~al.(2025)Semenov, Pagliardini, and Jaggi]{semenov2025benchmarking}
Andrei Semenov, Matteo Pagliardini, and Martin Jaggi.
\newblock Benchmarking optimizers for large language model pretraining.
\newblock \emph{arXiv preprint arXiv:2509.01440}, 2025.

\bibitem[Shalev-Shwartz and Ben-David(2014)]{shalev2014understanding}
Shai Shalev-Shwartz and Shai Ben-David.
\newblock \emph{Understanding machine learning: From theory to algorithms}.
\newblock Cambridge university press, 2014.

\bibitem[Shanno(1970)]{shanno1970conditioning}
David Shanno.
\newblock Conditioning of quasi-{N}ewton methods for function minimization.
\newblock \emph{Mathematics of Computation}, 24\penalty0 (111):\penalty0 647--656, 1970.

\bibitem[Team et~al.(2025)Team, Bai, Bao, Chen, Chen, Chen, Chen, Chen, Chen, Chen, et~al.]{team2025kimi}
Kimi Team, Yifan Bai, Yiping Bao, Guanduo Chen, Jiahao Chen, Ningxin Chen, Ruijue Chen, Yanru Chen, Yuankun Chen, Yutian Chen, et~al.
\newblock Kimi k2: Open agentic intelligence.
\newblock \emph{arXiv preprint arXiv:2507.20534}, 2025.

\bibitem[Tian et~al.(2025)Tian, Qiao, Liu, Jiang, Li, and Li]{tian2025survey}
Kaiyuan Tian, Linbo Qiao, Baihui Liu, Gongqingjian Jiang, Shanshan Li, and Dongsheng Li.
\newblock A survey on memory-efficient transformer-based model training in ai for science.
\newblock \emph{arXiv preprint arXiv:2501.11847}, 2025.

\bibitem[Tian et~al.(2023)Tian, Zhang, and Zhang]{tian2023recent}
Yingjie Tian, Yuqi Zhang, and Haibin Zhang.
\newblock Recent advances in stochastic gradient descent in deep learning.
\newblock \emph{Mathematics}, 11\penalty0 (3):\penalty0 682, 2023.

\bibitem[Tieleman(2012)]{tieleman2012rmsprop}
Tijmen Tieleman.
\newblock Lecture 6.5-rmsprop: Divide the gradient by a running average of its recent magnitude.
\newblock \emph{COURSERA: Neural Networks for Machine Learning}, 4\penalty0 (2):\penalty0 26, 2012.

\bibitem[Vaswani et~al.(2017)Vaswani, Shazeer, Parmar, Uszkoreit, Jones, Gomez, Kaiser, and Polosukhin]{vaswani2017attention}
Ashish Vaswani, Noam Shazeer, Niki Parmar, Jakob Uszkoreit, Llion Jones, Aidan Gomez, {\L}ukasz Kaiser, and Illia Polosukhin.
\newblock Attention is all you need.
\newblock \emph{Advances in Neural Information Processing Systems}, 30, 2017.

\bibitem[Vyas et~al.(2025)Vyas, Morwani, Zhao, Shapira, Brandfonbrener, Janson, and Kakade]{vyas2024soap}
Nikhil Vyas, Depen Morwani, Rosie Zhao, Itai Shapira, David Brandfonbrener, Lucas Janson, and Sham~M. Kakade.
\newblock {SOAP}: Improving and stabilizing shampoo using {Adam}.
\newblock In \emph{International Conference on Learning Representations}, 2025.
\newblock arXiv:2409.11321.

\bibitem[Wang et~al.(2018)Wang, Singh, Michael, Hill, Levy, and Bowman]{wang-etal-2018-glue}
Alex Wang, Amanpreet Singh, Julian Michael, Felix Hill, Omer Levy, and Samuel Bowman.
\newblock {GLUE}: A multi-task benchmark and analysis platform for natural language understanding.
\newblock In Tal Linzen, Grzegorz Chrupa{\l}a, and Afra Alishahi, editors, \emph{Proceedings of the 2018 {EMNLP} Workshop {B}lackbox{NLP}: Analyzing and Interpreting Neural Networks for {NLP}}, pages 353--355, Brussels, Belgium, November 2018. Association for Computational Linguistics.
\newblock \doi{10.18653/v1/W18-5446}.

\bibitem[Ward(2022)]{ward2022stochastic}
Rachel Ward.
\newblock Stochastic gradient descent: where optimization meets machine learning.
\newblock In \emph{Proc. Int. Cong. Math}, volume~7, pages 5140--5153, 2022.

\bibitem[Wen et~al.(2025)Wen, Hall, Ma, and Liang]{wen2025fantastic}
Kaiyue Wen, David Hall, Tengyu Ma, and Percy Liang.
\newblock Fantastic pretraining optimizers and where to find them.
\newblock \emph{arXiv preprint arXiv:2509.02046}, 2025.

\bibitem[Xie et~al.(2025)Xie, Wang, Reddi, Kumar, and Li]{xie2025structured}
Shuo Xie, Tianhao Wang, Sashank Reddi, Sanjiv Kumar, and Zhiyuan Li.
\newblock Structured preconditioners in adaptive optimization: A unified analysis.
\newblock In \emph{International Conference on Machine Learning}. PMLR, 2025.
\newblock arXiv:2503.10537.

\bibitem[Yen et~al.(2025)Yen, Si, Meng, Yu, Duvvuri, Dhillon, Hsieh, and Kumar]{yen2024lora}
Jui-Nan Yen, Si~Si, Zhao Meng, Felix~X. Yu, Sai~Surya Duvvuri, Inderjit~S. Dhillon, Cho-Jui Hsieh, and Sanjiv Kumar.
\newblock {LoRA} done {RITE}: Robust invariant transformation equilibration for {LoRA} optimization.
\newblock In \emph{International Conference on Learning Representations}, 2025.
\newblock arXiv:2410.20625.

\bibitem[Zellers et~al.(2019)Zellers, Holtzman, Bisk, Farhadi, and Choi]{zellers2019hellaswag}
Rowan Zellers, Ari Holtzman, Yonatan Bisk, Ali Farhadi, and Yejin Choi.
\newblock {H}ella{S}wag: Can a machine really finish your sentence?
\newblock In Anna Korhonen, David Traum, and Llu{\'i}s M{\`a}rquez, editors, \emph{Proceedings of the 57th Annual Meeting of the Association for Computational Linguistics}, pages 4791--4800, Florence, Italy, July 2019. Association for Computational Linguistics.
\newblock \doi{10.18653/v1/P19-1472}.

\bibitem[Zhang et~al.(2025)Zhang, Moniri, Nagwekar, Rahman, Xue, Hassani, and Matni]{zhang2025concurrence}
Thomas T. C.~K. Zhang, Behrad Moniri, Ansh Nagwekar, Faraz Rahman, Anton Xue, Hamed Hassani, and Nikolai Matni.
\newblock On the concurrence of layer-wise preconditioning methods and provable feature learning.
\newblock In \emph{International Conference on Machine Learning}. PMLR, 2025.
\newblock arXiv:2502.01763.

\end{thebibliography}
\end{document}